\documentclass[twoside,11pt]{article}
\pdfoutput=1 
\usepackage{amsthm}
 
\newtheorem*{theorem*}{Theorem}
\newtheorem{theorem}{Theorem}
\newtheorem{lemma}{Lemma} 
\newtheorem{proposition}{Proposition} 
\newtheorem{remark}{Remark}
\newtheorem*{remark*}{Remark}

\newtheorem{definition}{Definition}

\usepackage{jmlr}

\usepackage{amsmath}
\usepackage{bm}
\usepackage{algorithm}
\usepackage{algpseudocode}
\usepackage{mathtools}
\usepackage{wrapfig}
\usepackage{multicol}
\usepackage{xspace}
\usepackage{forloop}
\usepackage{float}
\usepackage{fancyhdr,graphicx}
\usepackage{thmtools}
\usepackage{thm-restate}
\usepackage{hyperref}
\usepackage{cleveref}
\usepackage{booktabs}       %
\usepackage{listings}

\usepackage[table,dvipsnames]{xcolor}

\newcommand{\defvec}[1]{\expandafter\newcommand\csname v#1\endcsname{{\mathbf{#1}}}}
\newcounter{ct}
\forLoop{1}{26}{ct}{
	\edef\letter{\alph{ct}}
	\expandafter\defvec\letter
}
\forLoop{1}{26}{ct}{
	\edef\letter{\Alph{ct}}
	\expandafter\defvec\letter
}

\newcommand{\Matern}{Mat\'ern\xspace}

\newcommand{\HM}{Hida-Mat\'ern\xspace}
\newcommand{\HMs}{Hida-Mat\'erns\xspace}
\newcommand{\MHM}{\text{MHM}}

\newcommand{\WienerProcess}{\mathcal{W}}
\newcommand{\statenoise}{\boldsymbol{\epsilon}}

\newcommand{\cctop}{H}
\newcommand{\given}{\mid}

\newcommand{\Ks}{\vK^S}
\newcommand{\fs}{\vf^S}
\newcommand{\gs}{\vg^S}
\newcommand{\zs}{\vz^S}

\newcommand{\Fhat}{\hat{F}}
\newcommand{\Pinf}{\vP_{\infty}}

\newcommand{\bPhi}{\boldsymbol{\Phi}}
\newcommand{\bLambda}{\boldsymbol{\Lambda}}
\newcommand{\balpha}{\boldsymbol{\alpha}}
\newcommand{\hIndexSet}{\mathcal{I}}

\newcommand{\A}{\vA}
\newcommand{\Q}{\vQ}

\newcommand{\vzero}{\boldsymbol{0}}

\newcommand{\field}[1]{\ensuremath{\mathbb{#1}}}

\newcommand{\reals}{\field{R}}
\newcommand{\complex}{\field{C}}

\newcommand{\naturalNumbers}{\field{N}}

\DeclareMathOperator*{\argmin}{\arg\!\min}
\DeclarePairedDelimiter{\norm}{\lVert}{\rVert}
\DeclarePairedDelimiter{\abs}{\lvert}{\rvert}
\DeclareMathOperator*{\cov}{cov}
\DeclareMathOperator*{\Expect}{\rm E} %

\DeclareMathOperator{\KL}{\mathcal{KL}}
\DeclareMathOperator{\vect}{vec}

\DeclarePairedDelimiter\floor{\lfloor}{\rfloor}

\usepackage{caption}
\DeclareCaptionLabelFormat{algnonumber}{Algorithm}
\definecolor{retroblue1}{cmyk}{0.89, 0.46, 0.24, 0.04}
\definecolor{retroblue2}{cmyk}{0.58, 0.15, 0.40, 0}
\definecolor{retropale1}{cmyk}{0, 0.72, 0.91, 0}

\begingroup
\catcode`_ \active
\catcode`^ \active
\gdef\newactiveunderscore
{\def_##1{\sb{##1\vphantom{\smash[b]{|}}}}}

\gdef^#1{\sp{#1\vphantom{\smash[t]{\big|}}}}
\endgroup

\newcommand{\modifiedsubsuper}{%
	\newactiveunderscore %
	\catcode`_ 12 %
	\ifnum\mathcode`^="8000 \else\edef\currentspmathcode{\the\mathcode`^}\fi
	\mathcode`^ "8000
	\catcode`^ 12 %
}%

\AtBeginDocument{\xdef\currentspmathcode{\the\mathcode`^}}

\ShortHeadings{\HM Kernel}{Dowling, Sokol, Park}
\firstpageno{1}

\begin{document}

\title{\HM Kernel}

\author{\name Matthew Dowling \email matthew.dowling@stonybrook.edu \\
       \addr Department of Electrical Engineering\\
       Stony Brook University\\
       Stony Brook NY, 11953, USA
       \AND
       \name Piotr A. Sok\'{o}\l{} \email piotr.sokol@stonybrook.edu \\
       \addr Department of Neurobiology and Behavior\\
       Stony Brook University\\
       Stony Brook NY, 11953, USA
	   \AND
	   \name Il Memming Park \email memming.park@stonybrook.edu \\
       \addr Department of Neurobiology and Behavior\\
       Stony Brook University\\
       Stony Brook NY, 11953, USA}

\maketitle

\begin{abstract}%
	We present the class of \textit{\HM} kernels, a canonical family of covariance functions that densely represents the entire space of stationary Gauss-Markov processes.
	It extends upon the \Matern kernels with flexible oscillatory components.
	Any stationary kernel, including the widely used squared-exponential and spectral mixture kernels, are either directly within this class or are appropriate asymptotic limits, demonstrating the generality of this class.
	Taking advantage of its Markovian nature, we show how to analytically represent a \HM process as a state-space model using only the kernel and its derivatives.
	In turn this allows us to perform Gaussian process inference more efficiently and side step the usual computational burdens.
	We further improve the numerical stability and reduce computational complexity by exploiting the structural properties of the state-space representation.
\end{abstract}


\section{Introduction}
The Gaussian process (GP) framework provides principled means to make inferences on functions~\citep{Rasmussen2005-mq}.
Endowed with a calibrated measure of uncertainty, GPs fit well within the Bayesian machine learning paradigm and can be embedded as a part of a broad class of models~\citep{bui_hieararchical_gps,rasmussen_classification_approx_gps,ng_gps_graphs_semisupervised}.
However, in practice the computational burden of using GPs typically limits their applicability to moderately sized datasets.
Scalable frameworks such as inducing point methods, specially structured kernels, streaming approaches, state-space formulation of GPs and so on seek to remedy this weakness by reducing the computational complexity of GP inferences~\citep{bui_streaming_gps_2017,titsias_sparse_variational_paper,Wilson2015-ak,sarkka_hartikainen_ssms_2010}.
Earlier literature on GPs in the 1950s and 60s was largely focused on their theoretical properties.
We discovered that many of these findings have practical implications for developing more scalable inference frameworks.
Building on pioneering works by Takeyuki Hida and others, we introduce the class of \HM kernels, which form a basis for translation invariant kernels and readily admit a state-space representation through which exact inference can be made in linear time.

A specific GP is characterized by its covariance function, or kernel; if its analytical form is given, its inspection reveals properties such as stationarity, periodicity, and differentiability.
However, beyond the second order statistical structure of the process, the kernel alone fails to rigorously quantify aspects such as Markovianity, sample path properties, and uniqueness of representation.
Building on P.~Levy's constructive formulation of GPs, or as it was called, a \textit{canonical representation}~\citep{levy_canonical_gps,levy_wiener_random_functions_1951},
T.~Hida was able to broadly generalize characteristics of GPs~\citep{hida_canonical_representation_1960,hida_gp_book}.

While in a practical sense many of these theoretical properties may be of little consequence, fully understanding the Markov properties of GPs can help alleviate the computational burden commonly associated with them.
For example, the GP with \Matern $\tfrac{1}{2}$ kernel (a.k.a. Ornstein-Uhlenbeck process) is well known to admit fast inference schemes thanks to its Markov property~\citep{stein_interpolation_textbook}.
As we will see in Section~\ref{sec:SSM}, there exist other ways of defining a generalized GP Markov property, which in turn, provide a clear avenue to formulating them in terms of a state-space model (SSM).
The SSM representation can then be used so that exact GP inference can be had in linear time given ordered data.
\begin{wrapfigure}[18]{r}{0.5\textwidth}
	\centering
	\includegraphics[width=0.47\textwidth]{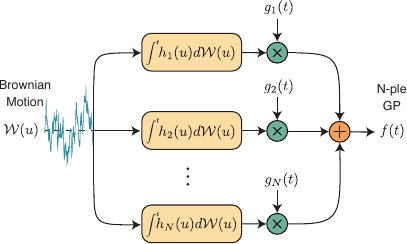}
	\caption{An $N$-ple GP can be thought of as the sum of $N$ additive processes, $f_i(t)$.  The deterministic functions $h_i(t)$ are integrated with respect to the Brownian motion, multiplied pointwise with $g_i(t)$ and then summed.}
	\label{label:NplGP}
\end{wrapfigure}

More than this though, we will show that the class of GPs with admissible state-space representations turns out to be very broad. In fact, \textbf{all stationary, real-valued, and finitely differentiable GPs}, which we will refer to as \HM GPs (H-M GPs), have canonical representations that must be linear combinations of basis functions initially derived by T.~Hida in~\citet{hida_canonical_representation_1960}.
The derived covariance functions, or \HM kernels, corresponding to these basis thus span the space of all such kernels governing H-M GPs.
Moreover, translation-invariant covariance functions that govern GPs not in this class can be approximated arbitrarily well by linear combinations of \HM kernels (Thm.~\ref{thm::main_paper::hida_materns_universal}).

Although state-space formulations of GPs have been examined extensively in recent literature, their formulation involves parametrizing a stochastic differential equation (SDE) whose stationary covariance matches that of the GP in question~\citep{solin_infinite_horizon_gp,sarkka_log_time_gp,solin_periodic_state_space,solin_thesis_2016}.
In contrast, our approach only requires determining all derivatives of the covariance function.
Furthermore, owing to Markov properties that will be discussed, the SSM formulation of any H-M GP is trivial to construct (Sec.~\ref{sec:SSM:Nple}).

To facilitate thinking beyond the second order structure of GPs, we re-introduce the importance of defining GP Markov properties through simple, yet enlightening examples (Sec.~\ref{sec:background}).
We then introduce the family of \HM kernels, present their universality with respect to $\mathcal{L}_2$ convergence in the space of translation-invariant kernels, how their Markovianity leads to simple SSM representations, and then how said representations open the door for computationally feasible GP inference (Sec.~\ref{sec:HM}).
We show that certain linear SDEs with stable dynamics admit a solution that lies within the \HM family, which consequently implies that the matrix exponential of such dynamics matrices has a closed form solution (Sec.~\ref{sec:SSM}).
Finally, we show examples of approximating arbitrary kernels through linear combinations of \HM kernels (Sec.~\ref{sec:optimization}), demonstrate how low-order \HM kernels extrapolate well on the Mauna Loa CO\textsubscript{2} data set, and illustrate the scalability of our approach in speed comparisons against state of the art methods (Sec.~\ref{sec:experiments}).

\section{Background: Canonical representation of stationary Gaussian processes}\label{sec:background}
\subsection{Two senses of Markovian GP}
Let $f(t)$ be a GP indexed over time, that is, for any finite time indices $(t_1, \ldots, t_n)$, the joint distribution of $(f(t_1), \ldots, f(t_n))$ is normal~\citep{Rasmussen2005-mq}.
The covariance function, $k(t, s) = \cov(f(t),f(s))$, and the mean function, $\Expect[f(t)]$, of a GP fully specifies its probabilistic structure.
An alternative constructive formulation of GPs by P. Levy led to the development of a \textit{canonical} representation for GPs as stochastic integrals with respect to a Brownian motion, from which new Markov properties were formulated.

\subsubsection{Markov in the restricted sense}
In the following, we consider univariate centered stationary GPs which are fully characterized by a translation invariant covariance kernel $k(t, s) = k(\tau)$ where $\tau = \abs{t - s}$.
One of the simplest examples of a Markovian GP is the Ornstein-Uhlenbeck (OU) process, with kernel $k(\tau) = \sigma^2 \exp(-\mu\tau)$~\citep{stein_interpolation_textbook}.
As we will see, it is \textit{how} one defines a Markov property for GPs that allows for greater insight into their behavior.
In this case, the OU process possesses the simplest Markov property in that,
$p(f(t) \mid \bm{\sigma}(f(s)); s < t) = p(f(t) \mid f(s))$, where $\bm{\sigma}(f(s))$\footnote{rigorously, $\bm{\sigma}(f(s))$ is the filtration of the process up until time $s$.} represents all information known about the process up to time $s$. The fact the OU process is Markov is easily identified by writing down its corresponding SDE and associated solution:
\begin{align}
	df(t) &= -\mu f(t) \, dt + \sigma d\WienerProcess(t)\\
	f(t) &= \exp(-\mu(t-s))f(s) + \sigma \int_s^t \exp(-\mu(t-\tau)) d\WienerProcess(\tau)
\end{align}
where $\WienerProcess(u)$ is the Wiener process (Brownian motion) ~\citep{hida_gp_book,oksendal_sde_intro_text,sarkka_applied_sdes_text}.
Note that the solution, which is a Gaussian process, depends only on the most recent known value $f(s)$ and not on the farther history $f(u), \, \forall u < s$~\citep{Jazwinski2007-yx}.
Now, it is easy to determine that the conditional distribution, $f(t) \mid f(s)$,
\begin{align}\label{eq:OU}
    p(f(t) \given f(s)) &= \mathcal{N}(
        f(t) \given
        \underbrace{{\color{retroblue1}k(t-s)k(0)^{-1}} f(s)}_{\substack{\text{mean}}},\,\,
        \underbrace{k(0) - k(t-s)^2 k(0)^{-1}}_{\substack{\text{variance}}}
        )
\end{align}

In this pedagogical example, we can view the kernel not only as a covariance function, but also as the \textcolor{retroblue1}{\textbf{\textit{operator}}} which propagates the process forward in time.
That said, the OU process has undesirable limitations -- it is mean square differentiable nowhere, thus its sample functions are very rough~\citep{Jazwinski2007-yx,stein_interpolation_textbook,levy_canonical_gps}.
This leads us to P. Levy's definition of $N$-ple Markov in the \textit{restricted sense}, extending the Markov property to GPs of higher order differentiability~\citep{,levy_wiener_random_functions_1951,hida_gp_book}.

\begin{definition}[$N$-ple Markov in the restricted sense]\label{def::n_ple_restricted}
	A GP, $f(t)$, is called $N$-ple Markov in the restricted sense if it is exactly $N-1$ times differentiable in mean square and
	\begin{equation}\label{eq::main_paper::restricted_markov}
	    p(f(t) \given \bm{\sigma}(f(s)); s \leq t) = p(f(t) \given f(s), f^{(1)}(s), \ldots, f^{(N-1)}(s))
	\end{equation}
	where $f^{(i)}$ denotes the $i^{\text{th}}$ mean square derivative of $f$.
\end{definition}

This definition has immediate consequences in reasoning about finitely differentiable GPs which we will show through a motivating example.
First, note that a GP and all of its mean square derivatives are jointly Gaussian~\citep{Jazwinski2007-yx,Sarkka2011-do} and define $k^{(p)}(\tau) \triangleq \frac{\partial^p}{\partial \tau^p} k(\tau)$ so that for a stationary GP, the covariance functions of the corresponding mean square derivative GP are
\begin{align}
	\cov(f^{(p)} (t), f^{(q)}(s)) = (-1)^{q} k^{(p+q)}(\tau) \quad\quad \text{with }\, \tau = \lvert t - s \rvert.
\end{align}

Now, consider a GP, $f(t)$, with the \Matern $\tfrac{3}{2}$ kernel (see Table~\ref{tbl:kernels}) parameterized by unit variance and length-scale, $k(\tau) = (1+\sqrt{3}\tau)\exp(-\sqrt{3}\tau)$~\citep{Rasmussen2005-mq}.
Since $k(\cdot)$ is once differentiable, this GP has only one mean square derivative, $f^{(1)}(t)$, so that $p(f(t) \given f(s), f^{(1)}(s); s < t) = p(f(t) \given f(s), f^{(1)}(s))$, making it a $2$-ple Markov GP in the restricted sense.
By the joint Gaussianity, the conditional density is
\begin{align}\label{eq::main_paper::matern_evolution_1d}
    p(f(t) \mid f(s), f^{(1)}(s))
	&= \mathcal{N}(f(t) \mid
	{\color{retroblue1}
	    \left[
		k(t-s), -k^{(1)}(t-s)
	    \right]
	    \Ks(0)^{-1}
	}
	\begin{bmatrix} f(s) \\ f^{(1)}(s)\end{bmatrix},\, q(t - s) )
\end{align}
where $q(t-s) = k(0) - \begin{bmatrix} k(t-s) ,& -k^{(1)}(t-s) \end{bmatrix} \Ks(0)^{-1}[ k(t-s), k^{(1)}(t-s)]^\top$ and
$\left[\Ks(0) \right] _{ij}= (-1)^j k^{(i+j)}(\tau)\rvert_{\tau=0}$.
Note how linear combinations of $k(t-s)$ and $k^{(1)}(t-s)$ fully describe how the process, $f(t)$, evolves over time.

We can generalize this to arbitrary N-ple Markov GPs in the restricted sense \textit{and} their mean square derivatives to understand how they evolve together over time.
Define a vector representation $\fs(t) \triangleq\modifiedsubsuper\begin{bmatrix} f(t), f^{(1)}(t), \ldots, f^{(N-1)}(t) \end{bmatrix}\overset{\lower.5em\hbox{$\top$}}{}$,
which, like the OU process earlier, is not differentiable (because of non-differentiable $f^{(N-1)}$).
The vector process $\fs(t)$ is a $1$-ple GP.
Now, $p(\fs(t) \mid \fs(t))$ can be determined so that
\begin{align}\label{eq::main_paper::matern_evolution}
	p(\fs(t) \mid \fs(s)) &= \mathcal{N}(\fs(t) \mid {\color{retroblue1}\Ks(t-s)\Ks(0)^{-1}}\fs(s), \vQ(t-s))\\
	\vQ(t-s) &= \Ks(0) - \Ks(t-s) \Ks(0)^{-1} \Ks(t-s)^\top
\end{align}
where $\left[\Ks(\tau) \right] _{ij}= (-1)^j k^{(i+j)}(\tau)$.
While Eq.~\eqref{eq::main_paper::matern_evolution} could be used to recursively determine the trajectory of $f(t)$, Eq.~\eqref{eq::main_paper::matern_evolution_1d} could not; retaining the most recent information about the process in conjunction with its mean square derivatives is key to inferring its future behavior.

\subsubsection{Markov in the Hida sense}
Consider a GP, $f(t) = f_1(t) + f_2(t)$, where $f_1(t)$ and $f_2(t)$ are both GPs with differently parameterized \Matern $\tfrac{3}{2}$ kernels.
Then, $f(t)$ is only once differentiable, however, it is not a $2$-ple GP in the restricted sense because Def.~\ref{def::n_ple_restricted} is not satisfied.
With Levy's definition being insufficient, we introduce T.~Hida's more general description of a GP Markov property, which we refer to as $N$-ple Markov (in the Hida sense).
\begin{definition}[$N$-ple Markov in the Hida sense~\citep{hida_gp_book}]\label{def::n_ple_hida}
	A GP, $f(t)$, is called an $N$-ple Markov GP if it admits the following filtered white noise representation:
	\begin{align}
		\label{eq::main_paper::gp_goursat_repr}
		f(t) &= \int_0^t F(t - u) d\WienerProcess(u)\\
		\label{eq::main_paper::gp_goursat_kernel_decomp}
		F(t - u) &= \sum_{i=1}^N g_i(t) h_i(u) \quad u \leq t \quad g_i, h_i: \, \reals \rightarrow \complex
	\end{align}
\end{definition}
$F(t-u)$ is called the ``\emph{canonical kernel}'' of the process in the literature, however, it should not be confused with a positive-semi-definite (covariance) kernel.
For clarity's sake, we will henceforth refer to $F(t-u)$ as the \textit{canonical filter} of the process.
Substitution of Eq.~\eqref{eq::main_paper::gp_goursat_kernel_decomp} into Eq.~\eqref{eq::main_paper::gp_goursat_repr} shows $f(t)$ is constructed as a linear combination of $N$ additive random processes by writing,
\begin{align}
    \nonumber
    f(t) &\triangleq \sum_{i=1}^N f_i(t) = \sum_{i=1}^N g_i(t) U_i(t)\\
    \nonumber
    U_i(t) &= \int^t h_i(u) d\WienerProcess(u).
\end{align}
Hence, the process $f(t)$ would be described as $N$-ple Markov in the Hida sense.  We see then that if a process is $N$-ple Markov (in the Hida sense) it may not be $N$-ple Markov in the restricted sense as Def.~\ref{def::n_ple_hida} does not require differentiability of the process.  Both of these definitions will be the groundwork for constructing appropriate state-space models that can be used for GP inference.

\subsection{The complete canonical basis}
A remarkable fact is that the canonical filter of \emph{any} stationary $N$-ple Markov GP lies in the span of a known set of basis functions.  The form of those basis functions as derived in~\citet{hida_gp_book} is given in the following theorem.
\begin{theorem}[Canonical Kernels (\citeauthor{hida_gp_book},~\citeyear{hida_gp_book}, p.~102)]\label{thm::main_paper::hida_canonical_basis}
	If $f(t)$ is a stationary $N$-ple Markov GP, then its canonical filter, $F(t-u)$, can be represented by a linear combination of the basis
\begin{align}
    F(t-u) &= \sum_{k=1}^m c_k (t-u)^{p_k} e^{-\mu_k(t-u)}
\end{align}
where $\mu_k \in \complex$ with $\text{Re}(\mu_k) > 0$, $p_k \in \naturalNumbers$, $\sum_{k=1}^m p_k = N$, and $c_k \in \complex$.
\end{theorem}
As an example, take $m=1$, $p_1=1$, $\mu_1 = \sqrt{3}$, and $c_1=12\sqrt{3}$,  then $F(t-u) = 12\sqrt{3} (t-u)\exp(-\sqrt{3}(t-u))$. By plugging this canonical filter into the equation above, we find that the stationary covariance of this process equals the \Matern $\tfrac{3}{2}$ kernel with unit length-scale and variance.
This invites the obvious question: what are the equivalent set of basis functions that span the space of admissible covariance functions over stationary, finitely differentiable GPs?
\begin{figure}[H]
	\centering
	\includegraphics[width=1.0\textwidth]{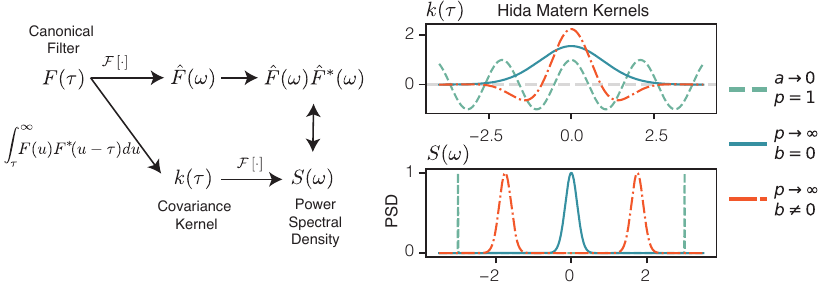}
	\caption{
	\textbf{Left}: Commutative diagram showing how the canonical filter is related to the covariance function and PSD.
	\textbf{Right}:  Kernels whom are limit points of \HM kernels.
	The \textcolor{retroblue2}{\textbf{cosine}} (spectral delta) kernel, \textcolor{retroblue1}{\textbf{squared exponential}} kernel, and \textcolor{retropale1}{\textbf{Gabor}} kernel and their PSD as limits of \HM kernels.
	}
	\label{fig:commutative}
\end{figure}
\section{The \HM Kernel}\label{sec:HM}
Leveraging the basis over canonical filters describing stationary and Markovian GPs, we can determine a corresponding basis over admissible covariance functions.  Let $f_p(t)$ be a GP with canonical filter given by a single basis as described in Thm.~\ref{thm::main_paper::hida_canonical_basis}, i.e. $F_p(\tau = t-u) = \tau^p e^{-\mu\tau}$, then $f_p(t)$ has power spectral density (PSD), $S_p(\omega)$, given by $S_p(\omega) = \Fhat_p(\omega)\Fhat_p(\omega)^\ast$ where $\Fhat_p(\omega)$ is the Fourier transform of $F_p(\tau)$.  In the case that $\mu$ is complex, then $S_p(\omega)$ would \textit{not} be the PSD of a real-valued GP as it would not be purely real and symmetric.

Keeping this in mind, we can determine a basis over real-valued covariance functions by isolating the real and symmetric parts of $S_p(\omega)$, giving us $S(\omega)$. Then, by invoking Bochner's theorem we can arrive at $k(\tau)$ through taking the inverse Fourier transform of $S(\omega)$~\citep{Rasmussen2005-mq}. The result is stated below and the full derivation is in Appendix~\ref{app::can_filter_deriv}.
\begin{proposition}
	The real-valued covariance function, and PSD corresponding to a canonical basis, $F_p(\tau) = \tau^p e^{-\mu \tau}$, with $\mu = a + jb$, are
    \begin{align}
        \label{eq::hida_kernel::kernel}
        k_{H,p}(\tau; a, b) &= \cos(b\tau) \, k_{\text{Mat}}\left(\tau;\, l=2\frac{\sqrt{p}}{a}, \nu=p+\tfrac{1}{2}\right)\\
        \label{eq::hida_kernel::psd}
        S_{H,p}(\omega; a, b) &= (p!)^2 \left[\left(\frac{1}{(\omega-b)^2 + a^2}\right)^{p+1} + \left(\frac{1}{(\omega+b)^2 + a^2}\right)^{p+1}\right]
        \end{align}
\end{proposition}
where $k_{\text{Mat}}(\tau; \, l, \nu)$ is the general \Matern covariance kernel of order $\nu$ and length-scale $l$.  Although simple, these kernels, which we call \textit{\HM} kernels, span the space of  stationary and finitely differentiable GP covariance functions.
Similar to a standard \Matern kernel the parameter $p$ controls the differentiability/smoothness, $a$ is the inverse length-scale, and $b$ controls the center of the PSD.  To conceptualize the Markov property better take $p=N$, and $b \neq 0$, then a GP with covariance function $k_{H,N}(\tau;\, a, b)$ will be $2N$-ple Markov in the Hida sense; however, when $b = 0$, then this covariance function coincides exactly with the \Matern kernels and such a GP would be $N$-ple Markov in the restricted sense.
The functional form of \eqref{eq::hida_kernel::kernel} has appeared in the literature before and we discuss this in Section~\ref{sec:history}.

There may be some confusion in the previous statement with regards to $f(t)$ being an $N$-ple GP when $b=0$ but $2N$-ple when $b \neq 0$ that an example may help clarify.
Take $f(t)$ to be a GP with canonical filter $F(\tau) = F_1(\tau) + F_2(\tau)$ and $F_1(\tau) = \tau \exp(-(a+jb)\tau)$, $F_2(\tau) = \tau \exp(-(a-jb)\tau)$.  Clearly, $F_1(\tau)$ is conjugate to $F_2(\tau)$ and individually each would be the canonical filter of a $2$-ple GP.  If we had that $b=0$ then both filters are identical and $F(\tau)$ is simply $2\tau\exp(-a\tau)$ making $f(t)$ a $2$-ple GP in the restricted sense however, summing them when $b \neq 0$ results in a GP that is $4$-ple in the Hida sense.

We can gain more expressive power from the \HM kernel by considering their linear combinations.  Though many strictly concern themselves with linear combinations such that the coefficients are always positive, it can be somewhat restrictive.  As long as the coefficients are chosen such that the resulting kernel is positive semidefinite, then negative coefficients are permitted~\citep{posa_difference_of_covariance_functions}.
Thus, the following theorem gives us good faith that we can achieve a respectable approximation of kernels that produce GPs infinitely differentiable in mean square using only a finite linear combination of \HM kernels parameterized by the same value of $p$ but varying inverse length-scales, $a$, and frequency parameters, $b$. %
\begin{theorem}[Mixture of stationary \HM kernels are dense.]
	\label{thm::main_paper::hida_materns_universal}
	For any fixed $p$, \HM kernels are dense in the space of square integrable functions, hence they are dense with respect to $\mathcal{L}_2$ convergence.
\end{theorem}

While those GPs that are infinitely differentiable in mean square have kernels that can not be approximated \textit{exactly} with a finite linear combination of \HM kernels, those that are finitely differentiable can.  GPs with covariance functions such as the squared exponential, spectral mixture, and cosine kernel fall into the class of infinitely differentiable GPs and in some literature are referred to as ``completely deterministic''. %
\begin{remark}
    Certain kernels such as the squared exponential do not fall within the class of GPs which may be represented by finite dimensional SDEs as they have a countably infinite number of derivatives and are analytic.  Such GPs are regarded as ``completely deterministic''~\citep{levy_canonical_gps}---meaning if observed for an infitessimal amount of time their future behavior is in theory completely known~\citep{stein_interpolation_textbook}.
\end{remark}

\subsection{The family of \HM GPs -- \HM Mixture (MHM) kernels}
With the \HM kernel defined, we can now consider the family formed by their linear combinations.
We define a \emph{mixture of \HM (MHM) kernels} as,
\begin{align}
	k_{H,N,\vp, \vc}(\tau) &= \sum_{i=1}^L c_i \, k_{H, p_i}(\tau; a_i, b_i)\\
	\vp &= \begin{pmatrix} p_1 & \cdots & p_L \end{pmatrix}^\top \in \reals^M, \quad \sum_{i=1}^L p_i = N,\\
	\vc &= \begin{pmatrix} c_1 & \cdots & c_L \end{pmatrix}^\top \in \reals^L, \quad c_i \neq 0, \,\,\, \forall i
\end{align}
where $\vp$ specifies the mixands smoothness, $\vc$ their respective weights, and $N$ the order of the Markov property in the Hida sense.  Note how based on the value of each $b_i$, if $N = \sum p_i$, then the process can be anywhere from $N$-ple Markov to $2N$-ple Markov in the Hida sense. Since \HM kernels with $p$ fixed form a universal class, the \HM mixture kernels do as well.

\section{\HM State Space Representations}\label{sec:SSM}
The intuition regarding Markov GPs that we developed earlier will help now in constructing a state-space representation of GPs that can be described by the \HM class of kernels~\citep{Jazwinski2007-yx}.  Though the state-space representation of GPs has been used successfully in the literature (see Sec.~\ref{sec:history}), construction of an appropriate SSM usually begins by correctly parameterizing a linear SDE whose solution has a stationary distribution coinciding  with the GP of interest~\citep{sarkka_hartikainen_ssms_2010,solin_infinite_horizon_gp,solin_thesis_2016}.  In contrast, the construction we present leverages the $N$-ple Markov property and allows for a state-space representation that only depends on the kernel and its derivatives.  To start, consider GPs that are $N$-ple Markov in the restricted sense, which for now, limits our scope to \HM kernels of order $N$ with $b=0$ (or \Matern kernels of order $\nu=p+\tfrac{1}{2}$ for $p \in \naturalNumbers$).
In the course of doing so, we will gradually expand upon this construction to handle the full family of \HM kernels.

\subsection{SSMs for \texorpdfstring{$\vN$}{N}-ple GPs in the restricted sense}\label{section::SSMs_restricted_Nple}
 Let $f(t)$ be an $N-1$ times differentiable, and stationary GP with kernel $k_{H,N}(\tau;\, a, 0)$, then $f(t)$ is an $N$-ple Markov GP.  By consolidating $f(t)$ and its mean square derivatives into the vector process
 \begin{equation}
 	  \fs(t) = \begin{pmatrix}f(t) & f^{(1)}(t) & \cdots & f^{(N-1)}(t)\end{pmatrix}^\top \in \reals^N
 \end{equation}
we have that $\fs(t)$ is a $1$-ple Markov GP in the restricted sense so that $p(\fs(t) \rvert \fs(s); s < t) = p(\fs(t) \rvert \fs(s))$.  Recognizing $\fs(t)$ is a multioutput GP  described by the kernel $\Ks(\tau)$, with $\left[\Ks(\tau) \right] _{ij}= (-1)^j k^{(i+j)}(\tau)$, we can write the conditional distribution explicitly:
\begin{align}
	\label{eq::main_paper::restricted_conditional_density}
	p(\fs(t) \rvert \fs(s)) &\sim \mathcal{N}(\fs(t) \rvert \A(\tau) \fs(s), \Q(\tau))\\
	\A(\tau) &= \Ks(\tau) \Ks(0)^{-1} \fs(s)\\
	\Q(\tau) &= \Ks(0) - \Ks(\tau) \Ks(0)^{-1} \Ks(\tau)^\top
\end{align}
where $\tau = |t-s|$ with $t > s$~\citep{alvarezComputationallyEfficientConvolved2011}. It is easy to see now that the conditional density in Eq.~\eqref{eq::main_paper::restricted_conditional_density} can be rewritten as a difference equation so that $\fs(t)$ is equivalently described by,
\begin{align}
	\label{eq::main_paper::difference_eq}
	\fs(t) &= \vA(\tau) \fs(s) + \boldsymbol{\epsilon}(\tau)\\
	\boldsymbol{\epsilon}(\tau) &\sim \mathcal{N}(\boldsymbol{\epsilon}(\tau) \rvert \boldsymbol{0}, \vQ(\tau))
\end{align}

The main object of interest $f(t)$ is easily extracted from $\fs(t)$ by letting $\vh = \begin{pmatrix} 1 & 0 & \cdots & 0\end{pmatrix}^\top$ so that $f(t) = \vh^\top \fs(t)$.  Thus, we can reason about the behavior of $f(t)$ through linear combinations of its past mean square derivatives or equivalently $\fs(t)$.  Much like our motivating examples earlier, $\Ks(\tau)$ alone characterizes how the process propagates forward through time. In order to motivate the practical utility of this construction we proceed by considering standard GP regression.

\subsection{GP Regression with finite SSMs}\label{section::main_paper::gp_regression_restricted}
In a GP regression setting with noisy observations, $\{y(t_i)\}_{i=1}^M$,  and stationary GP $f(t) \sim \mathcal{GP}(0, k(\tau))$ that is $N-1$ times differentiable, the standard generative model, as in~\citet{Rasmussen2005-mq}, can be recast as a linear Gaussian SSM using Eq.~\eqref{eq::main_paper::difference_eq} so that,
\begin{align}
	\fs(t_{i+1}) &= \vA(\tau_i) \fs(t_{i}) + \statenoise(\tau_i)\\
	y(t_i) &= \vh^\top \fs(t_i) + \nu
\end{align}
with $t_{i+1} > t_i \,\, \forall i$, $\tau_i = |t_{i+1} - t_i|$, and $\nu \sim \mathcal{N}(\nu \rvert 0, \sigma^2)$.  Since the aforementioned SSM is stable it admits a stationary covariance $\Pinf = \lim_{t\rightarrow \infty} \Expect[\fs(t)\fs(t)^\top]$~\citep{anderson_moore_state_space_book}. 
\begin{wrapfigure}[20]{r}{0.5\textwidth}
	\begin{center}
		\includegraphics[width=0.47\textwidth]{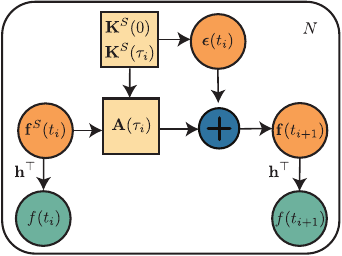}
	\end{center}
	\caption{Schematic representation depicting propagation of the process $\fs(t)$ forward in time.  The multioutput covariance kernel, $\Ks(\tau)$, completely determines the additive noise and propagates the mean value of the process.
	}
\end{wrapfigure}

Ensuring the process begins in the stationary state amounts to specifying $\fs(t_0) \sim \mathcal{N}(\vzero, \Pinf)$.  For SSMs, determining $\Pinf$ involves finding the solution of the continuous/discrete Lyapunov equation, however here $\Pinf = \Ks(0)$ (see Appendix~\ref{app::stationary_cov_hm_gp}).
In this form, the usual Kalman filtering and smoothing algorithms can be used to recover the posterior in $\mathcal{O}(MN^3)$ time, where $M$ is the number of data points.  

While the Kalman filtering algorithm is computationally efficient, ensuring that it is numerically stable can often be difficult.  As a result, ill numerical conditioning will be exacerbated by the fact that elements of $\Ks(\tau)$ are increasing in magnitude towards the bottom right, which we illustrate in Fig.~\ref{fig:condition_numbers}.   Approaches such as balancing or Nordsieck coordinate transformations, often used to combat this problem, formulate a surrogate SSM that can be used for equivalent inference~\citep{osborne_nordsieck_1966,  nordsieck_ode_transform_1962, sarkka_log_time_gp, kramer_stable_probabilistic_odes_2020}.  Similar in spirit to those approaches, we propose a \textit{correlation transform}, taking advantage of the SSMs formulation in terms of covariances.   Concretely, consider a linear transformation of the original process, $\zs(t) = \vC \fs(t)$ where
\begin{align*}
 [\vC]_{ii} = 1 \big/ \sqrt{\left[\Ks(0)\right]_{ii}}
\end{align*}
so that $\Ks_Z(\tau) \triangleq \text{cov}(\zs(t+\tau), \zs(t)^\top) = \vC \Ks(\tau) \vC^\top$. An alternative SSM is formed by substituting $\Ks_Z(\tau)$ for $\Ks(\tau)$ and rewriting the observation equation as $y(t_i) = \vh^\top \vC^{-1} \zs(t_i) + \nu$.  Whereas $\Ks(0)$ would be ill conditioned and thus cause problems when taking its inverse, $\Ks_Z(0)$ will have a significantly lower condition number, and we can expect the better numerical conditioning to provide more accurate inference.  After inference pertaining to $\zs(t)$ is made, properties of interest related to $\fs(t)$ are easily recovered.

\subsection{SSMs for general \texorpdfstring{$\vN$}{N}-ple GPs}\label{sec:SSM:Nple}
Having introduced how an SSM for a GP with \HM kernel of order $N$ when $b=0$ can be constructed, we are now in a position to consider the general case.  In order to make this jump, first observe that when $b \neq 0$ , a GP, $f(t)$, with kernel $k_{H,N}(\tau;\, a, b)$, will be a $2N$-ple process in the Hida sense but not the restricted sense.  We can see this most readily by breaking the cosine term into a sum of two complex exponentials so that,
\begin{align}
	\label{eq::main_paper::general_hida_complex_split}
	k_{H, N}(\tau; a, b) &= \tfrac{1}{2} e^{jb\tau} k_{H, N}(\tau; a, 0) + \tfrac{1}{2} e^{-jb\tau} k_{H, N}(\tau; a, 0)\\
	&= \text{Re}\{e^{jb\tau} k_{H, N}(\tau; a, 0) \}
\end{align}
which shows $k_{H,N}(\tau; a, b)$ is $N$ times differentiable, but the sum of complex conjugate \HM kernels of order $N$.  However, it is apparent that the representation given by Eq.~\eqref{eq::main_paper::general_hida_complex_split} is redundant due to the complex conjugacy.  Hence, we can consider a complex GP, $z(t)$, with kernel $k_z(\tau; a, b) = e^{jb\tau} k_{H, N}(\tau; a, 0)$, so that an equivalent description of $f(t)$ is given by the SSM,
\begin{align}
	\zs(t) &= \vA_z(\tau) \zs(s) + \boldsymbol{\epsilon}_z(\tau)\\
	f(t) &= \text{Re}\{\vh^\top \zs(t)\}\\
\boldsymbol{\epsilon}_z(\tau) &\sim \mathcal{N}(\boldsymbol{\epsilon}_z(\tau) \rvert \boldsymbol{0}, \vQ_z(\tau))
\end{align}
where
\begin{align}
	\A_z(\tau) &= \Ks_z(\tau) \Ks_z(0)^{-1}\\
	\Q_z(\tau) &= \Ks_z(0) - \Ks_z(\tau) \Ks_z(0)^{-1} \Ks_z(\tau)^\cctop
\end{align}
and similar to earlier, $\zs(t) = \begin{pmatrix}z(t) &  \cdots & z^{(N-1)}(t)\end{pmatrix}^\top$, and $[\Ks_z(\tau)]_{ij} = (-1)^{j} k_z^{(i+j)}(\tau)$.  This shows that we can reason about these particular $2N$-ple Markov GPs, in the Hida sense, with an $N$-dimensional state-space model by exploiting the complex-conjugate symmetries present.   Sans the necessity of having to work with complex numbers, this formulation lends itself to GP regression exactly as described in Section~\ref{section::main_paper::gp_regression_restricted}.
\subsection{SSMs for \HM mixtures}
Now that we have explored how to formulate the SSM describing a GP whose covariance function is an elementary \HM kernel, it is a trivial extension to consider the SSM formulation for the \HM mixture kernel, $k_{H,N,\vp, \vc}(\tau)$.  
\begin{align}
 \Ks_{H,N,\vp, \vc}(\tau) &=\text{diag}\begin{pmatrix}c_1\Ks_{H, p_1} (\tau) & c_2 \Ks_{H, p_2}(\tau) & \cdots & c_M \Ks_{H, p_M}(\tau)\end{pmatrix}
 \end{align}
so that $\Ks_{H,N,\vp, \vc}(\tau) \in \reals^{N\times N}$, and $\Ks_{H, p_i} (\tau)  \in \reals^{p_i \times p_i}$.  Thus, we can engineer  arbitrarily complex kernels as linear combinations of \HM kernels, yet work with them in the same manner by constructing an SSM with $\Ks_{H,N,\vp,\vc}(\tau)$.  For clarity sake later on if $f(t)$ is a GP with covariance function that is a mixture of \HM kernels we will say that $f(t)$ is an MHM (mixture of \HM kernels) GP denoted $f(t) \sim \MHM(f(t) \given \vp,\, \vc, \, N)$.

\section{\HM GPs and SDEs driven by Brownian motion}\label{sec:SDE}
It is well known that the solution of a linear SDE driven by Brownian motion is a Gauss-Markov process ~\citep{Jazwinski2007-yx,sarkka_applied_sdes_text,oksendal_sde_intro_text}.  In conjunction with our earlier discussion about Markovinity of univariate GPs, this observation raises the question, how does the state-space representation of a \HM GP relate to the solution of an SDE with the same stationary covariance. Delving into this question, we restrict ourselves to SDEs that can be written symbolically as
\begin{align}
	\mathcal{L} f(t) &= \frac{d}{dt}\WienerProcess(t)
\end{align}
where $\mathcal{L} = \sum a_i D^i$ is a differential operator of order $p$.  This formal representation can be transformed into an equivalent SDE much like in the works of ~\citep{sarkka_hartikainen_ssms_2010,sarkka_log_time_gp,solin_periodic_state_space} so that
\begin{align}\label{eq::main_paper::linear_sde}
	d\fs(t) = \vF \, \fs(t) \,dt + \vL \, d\WienerProcess(t)
\end{align}
with $\fs(t) \in \reals^N, \vL \in \reals^{N}$ and $\WienerProcess(t)$ a one dimensional Brownian motion. Taking $f(t)$ to be a \HM GP then we can consider its solution in terms of the preceding SDE and equivalent state-space representation so that we may develop stronger intuitions.  Writing the two representations side-by-side, for reasons that will become clear, we have
\begin{multicols}{2}
	\begin{equation*}
		\fs(t) = {\color{retroblue1}\bPhi(\tau)} \fs(s) + {\color{retropale1}\int_0^{\tau} \bPhi(\tau-u) \vL \, d\WienerProcess(u)}
	\end{equation*}\break
	\begin{equation*}
		\fs(t) = {\color{retroblue1}\Ks(\tau)\Ks(0)^{-1}}\fs(s) + {\color{retropale1}\epsilon(\tau)}
	\end{equation*}
\end{multicols}
where $\bPhi(\tau) = \exp(\vF \tau)$ is the state transition matrix of the system and $\tau = t-s$~\citep{Jazwinski2007-yx}.  Our first insight is that in order for both equations to be consistent $\color{retroblue1}\bPhi(\tau)$ and $\color{retroblue1}\Ks(\tau)\Ks(0)^{-1}$ must be equivalent so that $\Expect[\fs(t) \mid \fs(s)]$ is the same under either representation.   This tells us that we can find closed form solutions to the matrix exponential in terms of the derivatives of the \HM kernel.  

Moving forward it will be helpful to note that if $f(t)$ is a GP with kernel $k_{H,N}(\tau)$ then its SDE representation will have a dynamics matrix with one eigenvalue of multiplicity $N$. So, just to keep all of the ideas clear we have that a \HM GP with kernel $k_{H,N}(\tau)$ will be $N-1$ times differentiable in mean square, have a $N$ dimensional state-space representation, and its dynamics will have one eigenvalue of multiplicity $N$.  

Now, consider an arbitrary vector valued GP, $\gs(t)$, described by the SDE
\begin{align}\label{eq::main_paper::sde_multiplicity_n}
	d\gs(t) = \vG \gs(t)dt + \vM d\WienerProcess(t)
\end{align}
with $\vM \in \reals^N$, and the dynamics $\vG \in \reals^{N\times N}$ such that $\vG$ has one eigenvalue of multiplicity $N$.  If $f(t)$ is a \HM GP with kernel $k_{H,N}(\tau)$ then it in conjunction with its derivative processes satisfy a linear SDE $d\fs(t) = \vF \fs(t)dt + \vL d\WienerProcess(t)$.  In this case, $\vF$, like $\vG$, has one eigenvalue of multiplicity $N$, and $\vL = \begin{pmatrix} 0 & \cdots & 0 & \sqrt{k^{2N + 1}(\tau)}\rvert_{\tau=0}\end{pmatrix}$.  

The hyperparameters of $k_{H,N}(\tau)$ can be adjusted so that $\vF$ and $\vG$ have the same eigenvalue of multiplicity $N$ in which case they can be decomposed into their Jordan forms with $\vF = \vT^{-1} \vJ \vT$ and $\vG = \vS^{-1} \vJ \vS$.  We can form a surrogate process, $\vz^S(t)$, such that $\gs(t) = \vC \vz^S(t)$, where $\vC$ is chosen so that $\vS\vC = \vT$.  Now, $\vz^S(t)$ satisfies the SDE

\begin{align}
	\vz^S(t) &= \vC^{-1} \vG \vC \vz^S(t) + \vC^{-1} \vM d\WienerProcess(t)\\
	&= (\vS \vC)^{-1} \vJ (\vC \vS) \vz^S(t) + \vC^{-1} \vM d\WienerProcess(t)\\
	&= \vT^{-1} \vJ \vT \vz^S(t) + \vC^{-1}\vM d\WienerProcess(t)\\
	&= \vF \vz^S(t) + \vC^{-1}\vM d\WienerProcess(t)
\end{align}
which is equivalent to the SDE for the \HM GP $f(t)$ if $\vL = \vC^{-1}\vM$.  So far, this illustrates that $N$-dimensional SDEs with dynamics of multiplicity $N$ can be transformed by a change of coordinates to an SDE for a \HM GP and its derivative processes when $\vL = \vC^{-1} \vM$.

However, it is possible that $\vG$ has several eigenvalues of different multiplicity in which case we may wonder how to determine if that SDE can also be transformed into one that aligns with a \HM GP.  To fix the idea, take $\gs(t)$ as described in Eq.~\eqref{eq::main_paper::sde_multiplicity_n} and $\vG$ to have the Jordan decomposition $\vS^{-1} \vJ \vS$ with $\vJ = \vJ_{\lambda_1, m_1} \oplus \cdots \oplus \vJ_{\lambda_L, m_L}$.  We can construct a GP whose covariance function is a sum of \HM kernels parameterized so that its dynamics can be decomposed into $L$ Jordan blocks with the same multiplicity and eigenvalues as $\vG$.  Again, if $\vL = \vC^{-1}\vM$ with $\vC$ selected so that $\vS\vC = \vT$ then an appropriate coordinate change of $\gs(t)$ will produce an SDE that describes a \HM GP.  This gives us a necessary condition on a linear change of coordinates that transforms an arbitrary SDE into one with the interpretation that the vector process represents a \HM GP and its derivative processes given by the following Lemma.

\begin{lemma}
	\label{lemma::main_paper::sde_solutions_and_hida_materns}
	If given a linear and finite-dimensional SDE, written as
	\begin{align}\label{eq::appendix::linear_sde}
		d\gs(t) = \vG \, \gs(t) \,dt + \vM  \,d\WienerProcess(t)
	\end{align}
	where $\vM \in \reals^{N}$, $\mathcal{W}(t)$ a one dimensional Brownian motion process and dynamics $\vG \in \reals^{N\times N}$, with Jordan decomposition $\vG = \vS^{-1} \vJ \vS$, $\vJ = \vJ_{\lambda_1, m_1} \oplus \cdots \oplus \vJ_{\lambda_L, m_L}$ $\vM \in \reals^{N}$, then a coordinate change, $\vC$, mapping $\gs(t)$ to an SDE representing a \HM GP and its derivative processes exists if $\vL = \vC^{-1}\vM$ where $\vS\vC = \vT$.   Here, $\vT$ is taken so that the SDE for a \HM GP with $L$ mixands has dynamics matrix $\vF = \vT^{-1} \vJ \vT$.
\end{lemma}

In other words, a linear SDE with dynamics matrix $\vG$ having $L$ Jordan blocks may be equivalent to a linear transformation of the SSM formulation of a \HM GP with $L$ mixands.  Our second insight concerns the derivatives of a \HM covariance kernel.  When $f(t)$ is a \HM GP with kernel $k_{H,N}(\tau)$ we have the relation that $\Ks(\tau)\Ks(0)^{-1} = \exp(\vF\tau)$ for some $\vF\in\reals^{N\times N}$ whose eigenvalues lay in the left half plane.  By properties of the matrix exponential $\tfrac{d}{d\tau}\Ks(\tau)\Ks(0)^{-1} = \vF\exp(\vF\tau) = \vF \Ks(\tau)\Ks(0)^{-1}$, but we can use the fact that $\Ks(\tau)$ is a matrix of derivatives to understand what the form of $\vF$ is by inspection.  

First, note that taking the element wise derivative of $\Ks(\tau)$, for now not worrying about its last row, is equivalent to shifting its rows up by one position. Hence, $\vF$ must be a companion form matrix that essentially takes the element wise derivative of $\Ks(\tau)$ so that
\begin{align}
	\frac{d}{d\tau}\Ks(\tau) &= \vF \Ks(\tau)\\
	\vF &= \begin{bmatrix} 0 & 1 & 0 & \cdots & 0\\
		0 & 0 & 1 & \cdots & 0\\
		\vdots & \vdots & \vdots & \ddots & \vdots \\
		 & \rule[.5ex]{2.0em}{0.4pt} & \vc^\top & \rule[.5ex]{2.0em}{0.4pt} & \end{bmatrix}
\end{align}
What this reveals is that derivatives of $k_{H,N}(\tau)$ of order $N+1$ up to $2N+1$ are linear combinations of its first $N$ derivatives.  Since computing $\Ks(\tau)$ requires computing the $2N$ derivatives of $k_{H,N}(\tau)$ this offers another avenue for reducing computation as we only need to calculate the first $N$.
Let's work through a motivating example to make some of the ideas concrete.  Say we take the \Matern 3/2 kernel with unit variance and length-scale so that $k(\tau) = (1+\sqrt{3}\tau)\exp(-\sqrt{3}\tau)$.  Then, we have that
\begin{align}
	\Ks(\tau) = \begin{bmatrix} (1+\sqrt{3}\tau)\exp(-\sqrt{3}\tau) & -3\tau \exp(-\sqrt{3}\tau) \\ 3\tau \exp(-\sqrt{3}\tau) & 3(1-\sqrt{3}\tau)\exp(-\sqrt{3}\tau) \end{bmatrix}
\end{align}
Standard calculations yield
\begin{align}
	\left[\Ks(\tau)\right]^{-1} &= \begin{bmatrix} (1 - \sqrt{3}\tau) & \tau \\ -\tau & \tfrac{1}{3}(1+\sqrt{3}\tau)\end{bmatrix} \exp(2\sqrt{3}\tau)\\
	\Ks(\tau)^{(1)} &= \begin{bmatrix} 3\tau & 3(1-\sqrt{3}\tau)\\ -3(1-\sqrt{3}\tau) & 3(-2\sqrt{3} + 3\tau)\end{bmatrix} \exp(-\sqrt{3}\tau)
\end{align}
By the derivative property of the fundamental matrix solution, $\frac{d}{d\tau} \bPhi(\tau) = -\vF \bPhi(\tau)$, we can calculate the SDEs dynamics since $\bPhi(\tau) = \vA(\tau) = \Ks(\tau)\Ks(0)^{-1}$.  Plugging in reveals the dynamics
\begin{align}
	\vF &= -\bPhi(\tau)^{(1)} \bPhi(\tau)^{-1}\\
	&= -\Ks(\tau)^{(1)} \left[\Ks(\tau)\right]^{-1}\\
	&= \begin{bmatrix} 0 & 1 \\ 3 & 2\sqrt{3}\end{bmatrix}
\end{align}
which shows that $\vF$ is in companion form.  That fact will hold no matter the dimensionality of the SDE through recognizing that
\begin{align}
	\frac{d}{d\tau}\Ks(\tau) &= \vF \Ks(\tau)
\end{align}

which implies that higher order derivatives will be linear combinations of their lower order counterparts due to the companion form structure of $\vF$.  As a sanity check, one can consult the works ~\citet{solin_thesis_2016,sarkka_hartikainen_ssms_2010} to see that the same dynamics were found albeit through a much more algebraically demanding procedure.
\subsection{Multi-output \HM kernels}
In our discussion of SDEs we were more focused on the latent evolution of $f(t)$ and its derivatives, however, vector processes defined that way are valid multioutput GPs.  We will now consider those vector valued processes as a means to define a new class of multioutput GPs.

To consider multi-output GPs take a MHM GP, $f(t) \sim \MHM(f(t) \given \vp,\, \vc, \, N)$, and ignore the observation equation so that we are only concerned with $\fs(t)$.  As we said, $\fs(t)$ has an equivalent SDE formulation, with associated dynamics matrix $\vF$ having a particular Jordan block structure.  Taking $\gs(t) = \vX \fs(t)$, we have a new GP of the same dimension, and its SSM formulation becomes immediate as $\text{cov}(\gs(t+\tau), \gs(t)^\top) = \text{cov}(\vX \fs(t+\tau), \fs(t)^\top \vX^\top) = \vX \Ks(\tau)\vX^\top$ so that

\begin{align}
	\gs(t) &= \underbrace{\vX \Ks(\tau) \Ks(0)^{-1} \vX^{-1}}_{\substack{\vA_g(\tau)}} \gs(s) + \underbrace{\vX \boldsymbol{\epsilon}(\tau)}_{\substack{\boldsymbol{\epsilon}_g(\tau)}}
\end{align}
Now, we have that $\gs(t)$ is a multi-output MHM (MO-MHM) GP, i.e. $\gs(t): \reals \rightarrow \reals^d$.  Taking this one step farther, the observations can be modified so that rather than projecting the latent process onto one dimension via $\vh$, we project it to $\reals^D$ via $\vH\in\reals^{D\times d}$. The augmented SSM becomes
\begin{align}
	\gs(t) &= \vA_g(\tau)\gs(s) + \boldsymbol{\epsilon}_g(t)\\
	\vy(t) &= \vH \gs(t) + \boldsymbol{\nu}
\end{align}
Whereas before, we had latents $\fs(t)$ that represented the process and its mean square derivatives, the latents, $\gs(t)$, are their linear combinations as mapped by $\vX$.  We note that these linear transformations do not alter the Markov property of the process, they simply transform it to a different coordinate system.  From the discussion earlier, it is also clear that the dynamics matrix of the SDE describing $\gs(t)$ has the same Jordan block structure as that of $\fs(t)$ since the coordinate change only results in a new SSM whose dynamics are similar to those of $\fs(t)$ (where similar should be taken to mean the two are related through a similarity transform).

\section{Numerical and Algorithmic Properties}
In Section~\ref{section::main_paper::gp_regression_restricted} we noted that the structure of $\Ks_{H,p}(\tau)$ can lead to technical difficulties which make naive state-space inference infeasible.  The proposed correlation transform amends the ill numerical conditioning that manifests itself in a naive implementation and makes it possible to work with higher order \HM kernels.  There are additional precautions we may take that guarantee higher numerical stability as well as reductions in computational complexity.
\begin{wrapfigure}[15]{r}{0.5\textwidth}
	\begin{center}
		\includegraphics[width=0.47\textwidth]{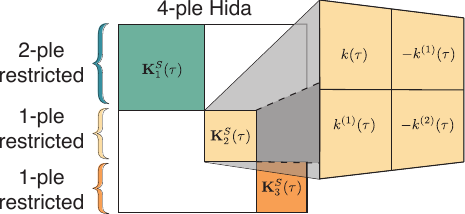}
	\end{center}
	\caption{Adding kernels that govern GPs Markov in the restricted sense result in a kernel over GPs that are Markov in the Hida sense. Forming appropriate SSMs amounts to considering the multioutput covariance formed from constituent blocks.}
\end{wrapfigure}

\subsection{Special structure of \texorpdfstring{$\Ks_{H,p}(\tau)$}{Ks(tau;H,p)}}
As the building blocks of $\Ks_{H,p}(\tau)$ are covariances between the process and its mean square derivatives, one may suspect that $\Ks_{H,p}(\tau)$ has some special structure that can be exploited.  In fact, many of the properties $\Ks_{H,p}(\tau)$ exhibits are similar to those of matrices explored in~\citet{strang_t_plus_ah}.

For any multioutput covariance matrix formed from a single \HM kernel we have that the $i^\text{th}$ off diagonal, up to appropriate sign flips, will only consist of the  $i^{\text{th}}$ derivative of the kernel evaluated at $\tau$. Hence, the multioutput covariance kernel is composed of only $2N-1$ unique elements, leading to memory requirements that scale linearly with the order of the kernel used.  This is especially useful when we consider that if we were working with observations not spaced uniformly then $\Ks(\tau)$, whose elements may require evaluation of complex equations, would need to be recomputed every time step.

Let's now consider the inversion of $\Ks_{H,N}(0)$.  If the correlation transform is properly used, then this inverse should not be a source of much trouble.  However, when $b=0$, the multioutput covariance kernel will be identically 0 at all indices $i,j$ such that $i+j$ is odd. To see this, note that the PSD of the \HM kernel is real and symmetric, and that $[\Ks_{H,N}(0)]_{ij}$ exactly coincides with the $(i+j)^\text{th}$ moment of $S(\omega)$, which is 0 for $i+j$ odd.  Recognizing this, elementary row and column operations can be used to transform the covariance matrix into the block matrix
\begin{align}
	\vR_L \cdots \vR_1 \, \Ks_{H,N}(0) \, \vC_1 \cdots \vC_L &= \begin{pmatrix}\vA & 0\\0 & \vB\end{pmatrix}
\end{align}
with $\vR_i$ and $\vC_i$ being elementary row/column operations.  The inverse of the altered matrix is then easily taken block by block.  Subsequent application of the inverse row and column operations return the desired inverse.  When $b\neq 0$ we do not have this sparsity present as we have chosen to work with the reduced, yet equivalent, SSM.  Understanding its structure can still help us to achieve improved numerical stability by enforcing properties we expect that numerical noise may break.  For example,  if $b \neq 0$ then $\Ks_{H,N}(0)$ should be purely imaginary at indices such that $i+j$ is odd. Hence, its inverse should retain this structure and if numerical noise causes these entries to become real, then they are easily masked.
\begin{figure}[H]
	\centering
	\includegraphics[width=1.0\textwidth]{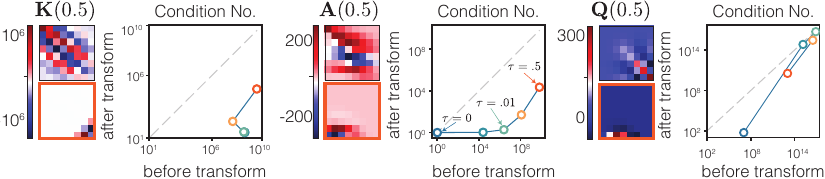}
	\caption{
	Illustrative figure showing the ill numerical conditioning if a coordinate transformation is not used for a \HM kernel $k_{H,8}(\tau; \, 1, 0)$. To the left of condition number plots are example matrices for $\tau=0.5$; in the orange box ($\color{orange}{\square}$) is the matrix before the correlation transform, above is after. Values of $\tau$ plotted are 0.0, 0.001, 0.01, 0.1, and 0.5. \textbf{left}: $\mathbf{K}(\tau)$ clearly benefits from the correlation transform, as evident in both the heatmap of the matrix and its condition number as a funciton of $\tau$. \textbf{middle}: $\vA(\tau)$ without the correlation transform has elements of increasing magnitude in the bottom left; the correlation transform amends this by creating a more homogenuous matrix in terms of relative magnitudes. \textbf{right}: $\mathbf{Q}(\tau)$ benefits the least from the correlation transform, however, it is clearly more balanced as seen from the heatmap.
	}
	\label{fig:condition_numbers}
\end{figure}
\subsection{Kalman updates}
In performing Kalman filtering over $M$ time-steps, the updated covariance, $\vP_m$ at time-step $m$, is not guaranteed to retain the positive semidefinite structure of a proper covariance matrix.  Often, the Joseph form of the covariance update or square root filtering can be used to ensure that positive semidefiniteness is not lost~\citep{anderson_moore_state_space_book}.  However, since $\vh$ is sparsely populated (i.e. a kernel which is the sum of $L$ \HM kernels will have $L$ nonzero values), the equations for the updated mean and covariance, $\vm_m$ and $\vP_m$, can be simplified for additional computational savings and superior numerical stability.

\begin{algorithm}
    \caption{Kalman filtering algorithm for GP regression with \HM kernels}
	\hspace*{\algorithmicindent} \textbf{Input} $\{t_i, y_i\}_{i=1}^M, \, \Ks(\tau), \, \sigma^2$\;
    \label{alg::main_paper::kalman_rts_algorithm}
    \begin{multicols}{2}

		\center{Kalman Filtering}
        \begin{algorithmic}[1]
            \State $\hIndexSet \leftarrow \text{where}(\vh == 0)$
            \State $\vP_0 \leftarrow \Pinf$
            \State $\vm_0 \sim \mathcal{N}(\vm_0 \mid 0, \Pinf)$
			\For{\texttt{$i = 1, \ldots, M$}}
				\State $\tau_i \leftarrow t_i - t_{i-1}$
				\State $\vA_i(\tau_i) \leftarrow \Ks(\tau_i) \Ks(0)^{-1}$
				\State $\vQ(\tau) \leftarrow \Ks(0) - \Ks(\tau_i)\Ks(0)^{-1}\Ks(\tau_i)^H$
				\State $\vP_i^{-} \leftarrow \vA_i(\tau_i) \vP_{i-1} \vA_i(\tau_i)^H + \vQ(\tau_i)$
				\State $\vm^{-}_i \leftarrow \vA_i \vm_{i-1}$
				\State $\alpha \leftarrow \left(\sum_{k,l \in \hIndexSet} \vP_i^{-}[k,l] + \sigma^2\right)^{-1}$
				\State $\beta \leftarrow \left(y_i - \sum_{k \in \hIndexSet} \vP_i^{-}[:, k]\right)$
				\State $\vm_i \leftarrow \vm_i^{-} + \alpha\beta \sum_{k \in \hIndexSet} \vP_i^{-}[:, k]$
				\State $\vP_i \leftarrow \vP_i^{-} - \alpha \sum_{k\in \hIndexSet} \vP_i^{-}[:,k] \vP_i^{-}[:,k]^\top$
		  	\EndFor
			\columnbreak

			RTS Smoothing
            \For{$i = M-1, \ldots, 1$}
				\State $\vG \leftarrow \vP_i \vA_i^\top  [\vP_i^{-}]^{-1}$
				\State $\vm_i \leftarrow \vm_i  + \vG\left(\vm_{i+1} - \vm^{-}_i\right)$
				\State $\vP_i \leftarrow \vP_i + \vG\left(\vP_{i+1} - \vP^{-}_{i}\right)\vG^\top$
			\EndFor
        \end{algorithmic}
    \end{multicols}
\end{algorithm}

In Algorithm~\ref{alg::main_paper::kalman_rts_algorithm}, which outlines the Kalman filtering/RTS smoothing algorithm to return the marginal posterior means and covariances in a GP regression setting, we can see the succint updates for $\vm_i$ and $\vP_i$.  The updated posterior mean, $\vm_i$, is a sum of the predicted mean, $\vm_i^-$, and the indexed columns of the predicted covariance, $\vP_i^-$.  Similarily, the updated covariance is a rank $L$ update of predicted covariance, where $L$ is the number of \HM mixands.
\section{Relation to Other Kernels}
Previously we noted that many  frequently used kernels either reside in the \HM family or can be approximated appropriately.  Take as an example the squared exponential kernel, it is well known that the standard \Matern family of kernels with smoothness parameter $\nu$ approaches a squared exponential as $\nu \rightarrow \infty$.  Remembering this, it is clear that the Spectral Mixture family of kernels is an asymptotic limit of a \HM mixture as the smoothness parameter $p_i \rightarrow \infty, \, \forall i$~\citep{Wilson2013-ud}.

The cosine kernel, $k(\tau) = \sigma^2 \cos(b\tau)$, is also not directly within the \HM family, but it is approached in the limit $a \rightarrow 0$ for $\sigma^2\,k_{H, p_i}(\tau; \, a, b)$. As another example, take the periodic kernel defined as $k(\tau) = \sigma^2 \exp(-\tfrac{2}{l^2}\sin^{-1}(\tfrac{b}{2}\tau))$, which can be expanded using its Taylor series representation~\citep{solin_periodic_state_space}
\begin{align}
	k_{PER}(\tau) &= \sigma^2 \exp\left( - \frac{2 \sin^2 (\omega_0\tfrac{\tau}{2})}{l^2}\right)\\
	&= \exp(-l^2) \sum_{q=0}^{\infty} \frac{1}{q!} \cos^q(\omega_0\tau)\\
	&= \exp(-l^2) \sum_{q=0}^{\infty}\sum_{v=0}^q {q \choose v}\frac{1}{q! 2^q} \text{Re}\left(e^{j\omega_0\tau(q-2v)}\right)\\
	&\approx  \exp(-l^2) \sum_{q=0}^{L}\sum_{v=0}^q {q \choose v}\frac{1}{q! 2^q} \text{Re}\left(e^{j\omega_0\tau(q-2v)}\right)\\
	&= \exp(-l^2) \sum_{q=0}^{L}\sum_{v=0}^q {q \choose v}\frac{1}{q! 2^q} k_{H,p}(\tau;\, a\rightarrow 0, b=\omega_0(q-2\nu))
\end{align}
Which illustrates how a kernel such as the periodic covariance can be decomposed such that it can be approximated by a \HM mixture.
\begin{table}[H]
	\begin{minipage}{.5\linewidth}
		\centering
		\begin{tabular}{@{}ll@{}}
		\toprule
		\textbf{kernel} & $k(\tau)$ \\ \midrule
		Squared Exp.   & $\sigma^2 \, \exp(-\frac{1}{2l^2}\tau^2)$\\

		Rational Quadr.    & $\left(1 + \frac{\tau^2}{2\alpha l^2}\right)^{-\alpha}$\\[1ex]

		Gabor.   & $\sigma^2 \cos(2\pi b) \, \exp(-\frac{1}{2l^2}\tau^2)$\\[1.525ex]

		Sinc & $\sigma^2 \text{sinc}(\Delta\tau) \cos(2\pi b \tau)$\\[1.525ex]
		\bottomrule
		\end{tabular}
	\end{minipage}
	\begin{minipage}{.5\linewidth}
		\centering
		\begin{tabular}{@{}ll@{}}
		\toprule
		\textbf{kernel} & $k(\tau)$ \\ \midrule
		Mat\'ern $(p+\tfrac{1}{2})$   & $\exp \left(-\frac{\sqrt{2 \nu} \tau}{\ell}\right) \frac{\Gamma(p+1)}{\Gamma(2 p+1)} \times$\\

		 & $\sum_{i=0}^{p} \frac{(p+i) !}{i !(p-i) !}\left(\frac{\sqrt{8 \nu} \tau}{\ell}\right)^{p-i}$\\

		Mat\'ern & $\frac{2^{1-\nu}}{\Gamma(\nu)}\left(\frac{\sqrt{2 \nu} \tau}{\ell}\right)^{\nu} K_{\nu}\left(\frac{\sqrt{2 \nu} \tau}{\ell}\right)$\\

		Spectral Mix.      & $\sum \sigma_i^2 \cos(2\pi b_i) \, \exp(-\frac{1}{2l_i^2}\tau^2)$\\
		\bottomrule
		\end{tabular}
	\end{minipage}
    \caption{Functional form of stationary kernels used throughout the paper. $\alpha > 0$, $l > 0$, $\nu > 0$, and  $K_\nu$ is the modified Bessel function of the second kind. See ~\cite{sinc_gp_bui_2018} for the Sinc kernel, ~\cite{Wilson2013-ud} for the Spectral Mixture kernel, and ~\cite{Rasmussen2005-mq} for further details on other kernels.}
    \label{tbl:kernels}
\end{table}

Now consider the LEG family of kernels introduced in ~\citet{cunningham_leg_kernel_2020}.  Take $\vz(t)$ to be the solution of a linear SDE driven by Brownian motion.  Linear transformation of $\vz(t)$ and addition of Gaussian noise gives an observed process $\vx(t)$ so that
\begin{align}
	\label{eq::main_paper::leg_latent}
	d\vz(t) &= -\tfrac{1}{2} \vG \, \vz(t) dt + \vN d\WienerProcess(s)\\
	\label{eq::main_paper::leg_observation}
	\vx(t) &= \vB\vz(t) + \bLambda\boldsymbol{\epsilon}(t)\\
	\boldsymbol{\epsilon}(t) &\sim \mathcal{N}(\boldsymbol{\epsilon}(t) \rvert \vzero, \vI)
\end{align}
In which case it is said that $\vx(t) \sim \text{LEG}(\vN, \vR, \vB, \bLambda)$ where $\vG = \vN\vN^\top + \vR - \vR^\top$.  The equivalence to a MO-MHM kernel is immediate from the discussion earlier on SDEs; the number of mixands determined by the Jordan block structure of $\vG$, with their hyperparameters determined by the eigenvalues of $\vG$. 

To make the equivalence more concrete, decompose $-\tfrac{1}{2}\vG$ into $\vC \vJ \vC^{-1}$ where $\vJ$ is a Jordan block matrix.  Say that $\vJ$ has $m$ blocks, each of size $p_i$, $i=1, \ldots, m$.  Then, a GP, $f(t)$, with \HM mixture kernel containing $m$ mixands, the $i-$th having order $p_i$, has an SSM formulation equivalent to the solution of a linear SDE with dynamics matrix $\vF$ such that $\vF = \vD \vJ \vD^{-1}$.  Taking $\vX\vD = \vC$, the SSM formulation of the vector process $\gs(t) = \vX \fs(t)$ will then be the solution of the SDE as defined in Eq.~\eqref{eq::main_paper::leg_latent}.  Further, setting $\vH = \vB$ and taking the covariance of $\boldsymbol{\nu}$ to be $\bLambda \bLambda^\top$ shows GPs defined either way are equivalent.

\subsection{Approximating arbitrary kernels}\label{sec:optimization}
Imagine a scenario where a kernel, $k_{\text{ref}}(\tau; \, \theta)$, has been designed, its hyperparameters chosen, and we wish to make GP inference under this kernel.  If the data set is small, exact GP inference can be used, however, scalability quickly becomes a concern and only approximate methods are viable.  When the kernel of interest can be approximated well by linear combinations of \HM kernels, then the SSM formulation presented is an appealing avenue.

Under this scenario, the question becomes how can the parameters of a \HM mixture, $k_{H, N, \vp, \vc}(\tau)$, be estimated so that it closely matches $k_{\text{ref}}(\tau; \, \theta)$. A practical, yet simple manner of estimating these hyperparameters is to minimize the squared loss between the reference kernel and the \HM mixture with respect to the mixtures hyperparameters. This results in the following optimization problem,
\begin{align}\label{eq::main_paper::l2_minimization}
	\nu = \argmin_\nu \int (k_{\text{ref}}(\tau; \, \theta) - k_{H, N, \vp, \vc}(\tau))^2 d\tau
\end{align}
where $\nu$ contains all hyperparameters of the \HM mixture to optimize. Through Parseval's Theorem, we can see that the objective in Eq.~\eqref{eq::main_paper::l2_minimization} not only minimizes the squared distance between the mixture and target kernels but also the squared distance of their PSDs~\citep{oppenheim_dsp_14}.
\begin{figure}[H]
	\centering
	\includegraphics[width=1.0\textwidth]{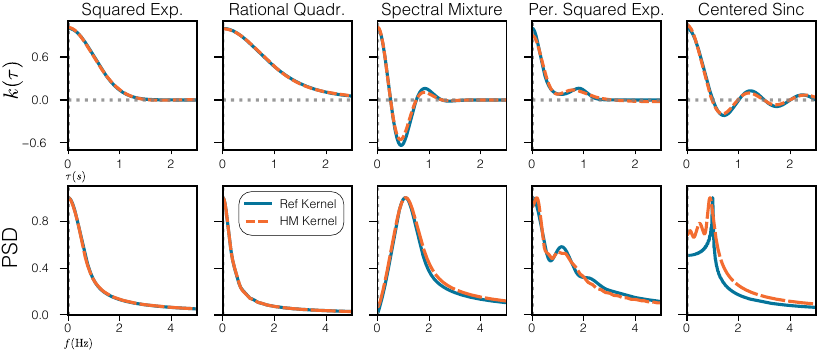}
	\caption{
	A \HM mixture of four mixands fit to various stationary kernels.  Each mixand is a second order \HM kernel and its parameters ($a$, $b$, $\sigma$) were optimized so that the squared distance between the mixture and reference kernel were minimized.
	}
	\label{fig:optimized_kernels}
\end{figure}
In Fig.~\ref{fig:optimized_kernels} are plotted \HM mixtures containing four mixands fit to various stationary kernels.  Though some of the reference kernels would have been approximated better through a single \HM kernel (take for example the Spectral Mixture kernel) the kernel functions and their PSDs, visually, are matched well.
\section{Historical Remarks}\label{sec:history}
In passing we have discussed the canonical representation of GPs.  Historically, many different representations and constructions of GPs have been used to understand their properties.  One of the most famous examples is the Kahrunen-Loeve expansion (KL) in which a stationary GP is decomposed as an infinite sum of randomly weighted basis functions.  Prior to Kahrunen and Loeve's work however, a similar representation was explored in the work of Kosambi but was not quite fully fleshed out~\citep{kosambi_1943}.

The representations we have mainly focused on are those that express GPs in terms of integral and differential operators.  Typically cited as one of the first works exploring this avenue is that of Doob, where he considers stationary and finitely differentiable GPs who can be formally described by a differential equation
\begin{align}
	\frac{d^N}{dt^N} f(t) - a_1 \frac{d^{N-1}}{dt^{N-1}} f(t) - \cdots - a_N f(t) = c \frac{d}{dt}\WienerProcess(t)
\end{align}
which should only be read symbolically as the derivative of Brownian motion does not exist~\citep{doob_1944}.  Subsequently, P. L\'evy being concerned with uniqueness of GP constructions developed what he called a canonical representation.  This representation allowed for specification of a stronger Markov property than that of Doob by dropping the restriction on stationarity of the process~\citep{levy_canonical_gps,hida_canonical_representation_1960,hida_gp_book}.

Tangentially related, and almost in parallel to Levy's development of a canonical GP representation was Woodbury's exploration of the connection between GP covariance functions and the Green's function of a suitably defined adjoint equation~\citep{woodbury_greens_fn_and_gps_1952}.  Approaching the problem similarily, T. Hida was able to determine an appropriate set of basis over GP covariance functions that are stationary and finitely differentiable~\citep{hida_canonical_representation_1960}.  

Following Hida's work, relationships between L-splines and realizations of sample functions from GPs defined by an appropriate SDE were made concrete in~\citet{wahba_spline_smoothing_state_space_1978}.
Viewing L-splines as realizations of GP sample functions, Weinert formulated L-spline fitting as making inference pertaining to an equivalent SSM formulation subsequently allowing for inference in $\mathcal{O}(N)$ time~\citep{weinert_part_1_smoothing_splines_1978,weinert_part_2_smoothing_splines_1980}.
Shortly after, Ansley considered the same SSM representation of a scalar GP although with a novel smoothing algorithm and further theoretical insights~\citep{ansley_spline_state_space_1987}.

More recently, in the context of machine learning, \citet{sarkka_hartikainen_ssms_2010} consider SDEs that have solutions whose stationary covariance coincides with the GP of interest.
Further work considered GPs' infinitely differentiable in mean square whose power spectral density is an analytic function; use of Pade approximations facilitated forming approximate SSMs for the purpose of posterior inference~\citep{karvonen_pade_approximations_15}.
Building more on these concepts later works were able to determine an SDE parameterization for periodic covariances, leverage the SSM representation in conjuction with approximate inference for fast inference in non-conjugate models, and used infinite horizon updates of the Kalman filter and Rauch-Tung-Striebel (RTS) smoother to ameliorate computation time in models with higher dimensional latent spaces~\citep{solin_periodic_state_space,solin_vi_ssm_gp_20,solin_infinite_horizon_gp}.
In~\citet{samoGeneralizedSpectralKernels2015} generalized spectral mixture kernels were explored which include the functional form identical to the \HM kernel, under the name \textit{Spectral Mat\'ern} kernel.
However, they did not touch upon the criticial implication that the \HM kernel forms a complete basis over finitely differentiable and stationary GPs.
In addition, the authors extended upon that work in~\citet{samo_p_markov_2015} where they explore consequences of the GP Markov property, but only in the restricted sense.
Only exploring the Markov property in the restricted sense leaves open questions about how the Markov property can be explained or reasoned about once GPs are added together or more pathological cases where differentiability of the GP arise.

\section{SSMs for GPs with Multivariate Input}
Thus far we have explored the SSM representation for GPs where the index variable is a scalar.  In this setting defining an ordering over the index variable that defines a Markov property is simple and intuitive; at time $t_0$ if $t_1 > t_0$ then $f(t_1)$ is the future of the process while if $t_2 < t_0$ then $f(t_2)$ is the past.  When we consider $\mathcal{X} \subset \reals^D$ with $D > 1$ it is not immediately obvious how to similarily define an appropriate Markov property for elements of $\mathcal{X}$.  

In ~\cite{pitt_multidimensional_markov_71}, explored was the notion of a GP Markov property where the index variable is not a scalar. Let $\mathcal{M}_i$, $i \in \naturalNumbers$ be $D-1$ dimensional manifolds in $D$ dimensional euclidian space such that if $\vx_i \in \mathcal{M}_i$ and $\vx_{i+1} \in \mathcal{M}_{i+1}$ then $\lvert \vx_i \rvert < \lvert \vx_{i+1} \rvert$.  Note that with concentric manifolds defined as such that we must have $\text{hull}({\mathcal{M}_i}) \subset \text{hull}({\mathcal{M}_{i+1}})$. Intuitively then, we can see how notions of `past' and `future' might be defined by viewing observations laying along any of $\mathcal{M}_1, \ldots, \mathcal{M}_{i-1}$ as the past with respect to $\mathcal{M}_i$.

\begin{figure}[H]
	\centering
	\includegraphics[width=1.0\textwidth]{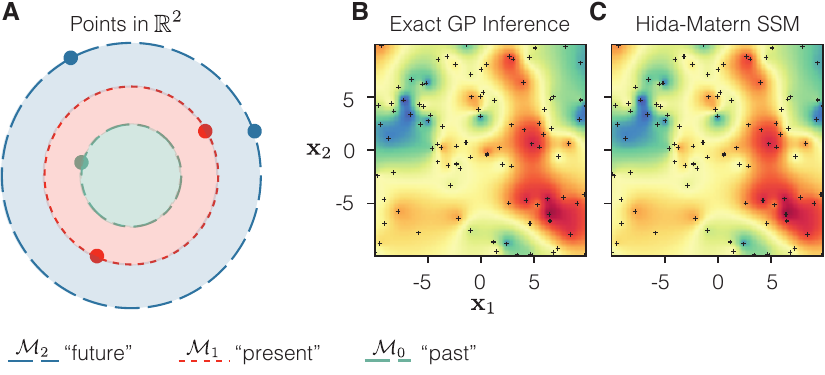}
	\caption{\textbf{A)} Cartoon illustration of the notions of `past', `future', and `present' in $\reals^2$. \textbf{B)} Black markers (\textbf{+}) depict spatial observations in $\reals^2$.  Data is generated from the two dimensional \Matern $\tfrac{3}{2}$ kernel, $k(\vx) = (1+\sqrt{3} \Vert \vx \Vert_2)\exp(-\sqrt{3}\Vert\vx\Vert_2)$, with points sampled on $[-5,5]^2$. First, we compute the posterior using exact GP inference over the spatial domain; plotted is the posterior mean. \textbf{C)} Inference over the same data as in panel \textbf{B}, but this time the SSM representation was used in combination with Kalman filtering/RTS smoothing to recover the posterior.  We see the posterior recovered under the SSM approach is the same as that under exact GP inference.}
	\label{fig::main_paper::multi_d_hm}
\end{figure}

Hence, if we have points laying in $D$ dimensional space it is easy to define concentric series of manifolds $\{\mathcal{M}_i\}$ with $\mathcal{M}_i = \{\vx \, : \, |\vx| = c_i\}$ so that for a simple Gaussian Markov process $p(f(\mathcal{M}_i) \rvert f(\mathcal{M}_{i-1}), \ldots, f(\mathcal{M}_1)) = p(f(\mathcal{M}_i) \rvert f(\mathcal{M}_{i-1}))$ .  To lay down some notation, define a differential operator $\mathcal{D}^{\balpha}$, taken to mean $\mathcal{D}^{\balpha}f(\vx) = \frac{\partial^M}{\partial^{\alpha_1} \cdots \partial^{\alpha_D}} f(\vx)$, with $\balpha \in \reals^D$ and each $\alpha_i \in \naturalNumbers$ such that $M = \sum_i \alpha_i$.

Let's play the same game as we did in the univariate case and see how we can take advantage of Pitt's Markov property to achieve fast inference when the index variable is no longer scalar.  Going back to our favorite \Matern $\tfrac{3}{2}$, its multivariate analogue is $k(\vx) = (1+ \sqrt{3}\Vert \vx\Vert_2) \exp(-\sqrt{3} \Vert \vx \Vert_2)$ which has three partial derivatives, $\frac{\partial}{\partial \vx_1}k(\vx)$, $\frac{\partial}{\partial \vx_2}k(\vx)$, and $\frac{\partial^2}{\partial \vx_1 \partial \vx_2}k(\vx)$.  Once these partial derivatives are calculated we can then form the multioutput covariance $\Ks(\vx)$.
\begin{align}
	\Ks(\vx) &= \begin{bmatrix} k(\vx) & \frac{\partial}{\partial \vx_1}k(\vx) & \frac{\partial}{\partial \vx_2}k(\vx)\\
	\frac{\partial}{\partial \vx_1}k(\vx) & \frac{\partial^2}{\partial \vx_1^2}k(\vx) & \frac{\partial^2}{\partial \vx_1 \vx_2}k(\vx)\\
	\frac{\partial}{\partial \vx_2}k(\vx) & \frac{\partial^2}{\partial \vx_1 \partial \vx_2}k(\vx) & \frac{\partial^2}{\partial \vx_2^2}k(\vx)\end{bmatrix}
\end{align}

from which we have that $\fs(\vx) = \Ks(\vx) + \boldsymbol{\epsilon}(\vx)$ where $\fs(\vx) = \begin{bmatrix} f(\vx) & \frac{\partial}{\partial \vx_1}k(\vx) & \frac{\partial}{\partial \vx_2}k(\vx)\end{bmatrix}^\top$.  If we extract $f(\vx)$ from $\fs(\vx)$ via $f(\vx) = \vh^\top \fs(\vx)$ just as in the scalar input case, then all we have discussed so far carries over quite easily; for example, we could form the standard generative model and easily perform GP regression in the case we have observations that can be modeled as $y(\vx) = \fs(\vx) + \nu$ where $\nu$ is Gaussian noise.

In Fig.~\ref{fig::main_paper::multi_d_hm} we can see the result of performing posterior inference over $\reals^2$ with the two dimensional input \Matern $\tfrac{3}{2}$ with exact GP inference as well as the SSM formulation.  Inspecting the posterior mean, we can see that the posterior under the SSM formulation is faithful to the posterior under exact GP inference.  One downside to performing spatial inference in this manner is that observations will need to be sorted according to the $\mathcal{L}_2$ norm of their location; however, even with sorting and making use of Kalman filtering/smoothing the computational cost is still minute compared to naive GP inference.

It is worthwhile noting for historical purposes that there are works outside of Pitt's that consider a Markov property for Gaussian random fields with multidimensional support.  Preceding Pitt's work is that of~\citet{mckean_brownian_motion_several_dim_63} who like Levy explores the characterization of a Markovian property for processes with an odd dimensional index variable, albeit taking a different approach.  It is Pitt's work that builds on McKean's by succesfully characterizing the Markov property of random fields with a support of arbitrary dimension.  Mentions of these works as well as others exploring a Markov property for processes with multidimensional index are briefly summarized in the text of~\citet[Appendix]{adler_random_fields}.  Hida in pursuit of a general theory for representation of white noise processes also considered those defined over a multivariate domain, the curious reader can consult the text~\citet{hida_sisi_white_noise_random_fields} for a fully fleshed out presentation similar to that of his theory of canonical representations of univariate GPs.  Also worth mentioning is the paper,~\citet{lee_1990_random_fields}, that goes more in depth with regards to functional forms of covariance kernels over manifold domains.

\section{Experiments}\label{sec:experiments}
\subsection{Mauna Loa Carbon Dioxide}
Here, we examine the predictive capabilities and expressivity of the \HM family by applying it to the popular Mauna Lua $\mathrm{CO_2}$ dataset~\citep{Rasmussen2005-mq}. The long upward trend present in this data as well as the yearly periodic component mean that an appropriate linear combination of \HM kernels will need to be constructed.  To demonstrate that a low order SSM is sufficient for this dataset we construct a simple kernel that is the sum of two order $3$ \HMs, i.e.
\begin{align*}
	k_{H, 6}(\tau) &= c_1 \, k_{H, 3}(\tau;\, a_1, b) + c_2\, k_{H, 3}(\tau;\,a_2, 0)
\end{align*} 
with $c_1=0.05^2$, $a_1=1 / 25$, $b=2\pi$, and $c_2=2.3^2$, $a_2=1/100$, i.e. one kernel which is decaying periodic with a short lengthscale to capture the yearly periodic trend and another which is not periodic with a long lengthscale to capture the linear trend.  This setup is similar to those commonly used in the literature, where the first additive kernel would usually be a Mat\'ern/Squared Exponential multiplying the periodic covariance function~\citep{Rasmussen2005-mq, sarkka_log_time_gp, solin_periodic_state_space}.
\begin{figure}[H]
	\centering
	\includegraphics[width=1.0\textwidth]{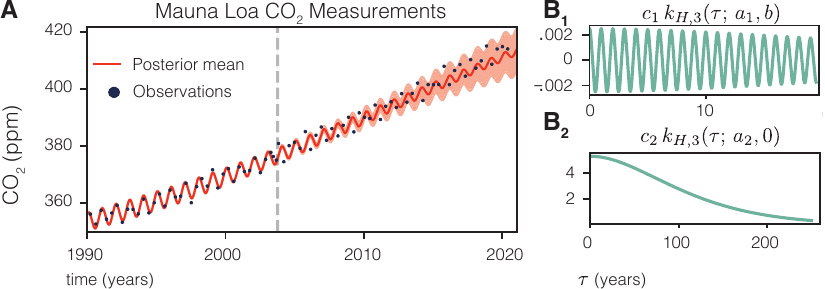}
	\caption{\textbf{A)} Mauna Loa dataset fit until 2004, after which only predictions were made. \textbf{B\textsubscript{1})}   The third order \HM kernel fit with a periodic component. \textbf{B\textsubscript{2})}  The third order \HM kernel fit with no periodic component.}
    \label{fig::main_paper::mauna_lua}
\end{figure}
The GP is fit using the data from 1974-2004, and predictions are made for the time window from 2004-2020 as shown in Fig.~\ref{fig::main_paper::mauna_lua}.  From the fit, we can see that the sum of these two low order Hida-Mat\'ern kernels form a covariance function such that the resulting GP inference can make predictions that capture both the seasonal and linear trends in the data.

\begin{figure}[H]
	\centering
	\includegraphics[width=1.0\textwidth]{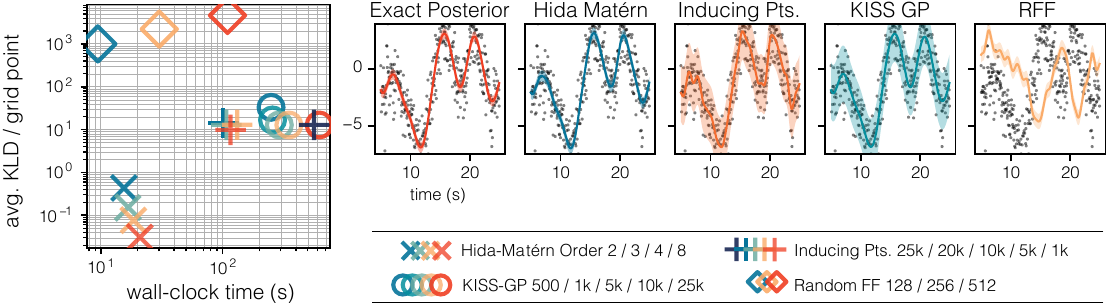}
	\caption{
	\textbf{Left}: Wall clock time versus average KLD, $\KL(q(f(t)) \rvert\lvert p(f(t)))$ with $p(f(t))$ the true posterior, per bin point between each method and the true posterior. \textbf{Right}: Posterior fits for the best parameter settings from each model class -- we see that in comparison to inducing points and KISSGP, although a low order \HM was used it does not overestimate the posterior uncertainty. Numbers in the legend indiciate the total number of (uniformly spaced) inducing points. 
	}
    \label{fig::main_paper::large}
\end{figure}

\subsection{Scalability of SSM representation}
From a practical standpoint, one of the most useful implications that arises from the $N$-ple GP Markov characterization is the ability to quickly formulate a corresponding state-space model amenable for inference.  As a result, it is trivial to form appropriate state-space models that can be used to recover exact GP inference over large datasets where naive GP regression would be impractical due to the cubic scaling of the computational complexity.

We now consider a toy dataset containing 50,000 observations, with uniform spacing of 0.05, that are generated  according to a prior GP whose covariance function is the sum of two spectral mixture kernels, i.e.,~\citep{Wilson2013-ud}
\begin{align}
	k(\tau) &= k_{\text{SM1}}(\tau; \, c_1, l_1, \omega_1) + k_{\text{SM2}}(\tau; \, c_2, l_2, \omega_2)\\
	&\text{with} \quad k_{\text{SM}i} = c_i \exp\left(-\frac{\tau^2}{2l_i^2}\right)\cos(\omega_i \tau)
\end{align}
where the hyperparameters are $c_1=1.5^2$, $l_1=2.0$, $\omega_1=2\pi\cdot0.01$, and $c_2=1.5^2$, $l_2=2.0$, $\omega_2=2\pi\cdot0.05$.  For comparison, we consider SVGPR, KISSGP, as well as random Fourier features~\citep{titsias_sparse_variational_paper,Wilson2015-ak,random_fourier_rahimi}. Experiments for SVGPR and KISSGP are ran using GPyTorch with all hyperparameter optimization done before computing the wall-clock time of the calculation for the posterior distribution over the grid~\citep{gpytorch_python}.  To subserviate any numerical difficulties that would arise from calculating the KLD between the posterior under each method and exact GP inference we instead consider the average KLD per grid point, or average marginalized KLD.

As the data points are distributed on a uniform grid we expect that inference using the SSM and \HM kernels should be the fastest as $\Ks(\Delta)$ only need be computed once.  Indeed, this is the case and it also has the lowest marginal KLD to the true posterior even though there is a model mismatch (as the Spectral Mixture is only an asymptote of the \HM family).

\section{Conclusion}\label{sec:conclusion}
We showed how viewing GPs through the lens of their Markov property has both theoretical and practical consequences.
We reintroduced results from Hida that all finitely differentiable stationary GPs are Markovian and their kernel must admit a decomposition in terms of linear combinations of the derived \HM kernels.
As a consequence, \HM GPs can be simply rewritten as a linear Gaussian state-space model.
The SSM representations enabled us to make exact GP inference in linear time for any 1-dimensional stationary GP whose kernel is in the \HM family.
As a by product of the admitted SSM representation, we also fleshed out connections to SDEs whose solutions are GPs, showing that the fundamental matrix solution of those systems has a closed form representation.
Finally, we showed many commonplace kernels either reside directly within the \HM family or can be seen as appropriate asymptotic limits.
The \HM kernel provides a unifying framework that bridges linear models used in the statistical signal processing literature and the nonlinear kernel methods in the machine learning literature.

\newpage
\appendix
\section*{Appendix}
\addcontentsline{toc}{section}{Appendix}
\renewcommand{\thesubsection}{\Alph{subsection}}

\subsection{Derivation of the Hida-Mat\'ern kernel starting from the basis of canonical filters}\label{app::can_filter_deriv}
Let $f_p(\tau)$ be a canonical filter of order $p$, i.e.,
\begin{align}
	\label{eq::appendix::canonical_filter_p}
	f_p(\tau)  &= \tau^p \exp(-\mu\tau)
\end{align}
the covariance function, $r_p(\tau)$, of a process described with such a canonical filter is
\begin{align*}
	r_p(\tau) = \int_{\tau}^{\infty} f_p(t) f_p^{\ast} (t - \tau) dt 
\end{align*}
where $^\ast$ denotes complex conjugation.  The PSD, $S(\omega)$, of $r_p(\tau)$ is related to the fourier transform of $f_p(\tau)$ so that we have $S(\omega) = F_p(\omega) F^{\ast}_p(-\omega)$ where $F_p(\tau) = \mathcal{F}\left[f_p(\tau)\right]$.  For convenience we proceed by working with the Fourier cosine and Fourier sine transforms denoted by
\begin{align*}
	F_{C,p}(\omega) &= \mathcal{F}^C\left[f_p(\tau)\right] = \int_0^\infty f_p(\tau) \cos(\omega\tau)d\tau\\
	F_{S,p}(\omega) &=\mathcal{F}^S\left[f_p(\tau)\right] = \int_0^\infty f_p(\tau) \sin(\omega\tau)d\tau
\end{align*}
so that $F_p(\omega) = F_{C,p}(\omega) - j F_{S,p}^S(\omega)$~\citep*{oppenheim_dsp_14}. In general by working with this basis we are not guaranteed that the resulting power spectral density of the covariance function will be real and symmetric.  To enforce this constraint  the imaginary part of $F_p(\omega)F^\ast_p(\omega)$ needs to be isolated.  Through the Fourier sine and cosine transformations we have
\begin{align}
	S(\omega) &= (F_{C,p}(\omega) + j F_{S,p}(\omega))(F_{C,p}(\omega) + j F_{S,p}(\omega))^\ast\\
	&= F_{C,p}(\omega)F_{C,p}^{\ast}(\omega) + F_{S,p}(\omega)F_{S,p}^{\ast}(\omega) - jF_{C,p}(\omega)F_{S,p}^{\ast}(\omega) + jF_{S,p}(\omega)F_{C,p}^{\ast}(\omega)\\
	&= \norm{F_{C,p}(\omega)}^2 + \norm{F_{S,p}(\omega)}^2 - jF_{C,p}(\omega)F_{S,p}^{\ast}(\omega) + jF_{C,p}^{\ast}(\omega)F_{S,p}^{\ast}(\omega)
\end{align}
Now,  $\norm{F_{C,p}(\omega)}^2 + \norm{F_{S,p}(\omega)}^2$ must be isolated.  For analytic purposes, this will be more easily achieved by considering the following identity,
\begin{align}
	\label{eq::appendix::expanded_fourier_representation}
	2\left(\norm{F_{C,p}(\omega)}^2 + \norm{F_{S,p}(\omega)}^2\right) = &(F_{C,p}(\omega) + j F_{S,p}(\omega))(F_{C,p}(\omega) + j F_{S,p}(\omega))^\ast\\
	\nonumber
	&+ (F_{C,p}(\omega) - j F_{S,p}(\omega))(F_{C,p}(\omega) - j F_{S,p}(\omega))^\ast
\end{align}
With a canonical filter given in the form of Eq.~\eqref{eq::appendix::canonical_filter_p} we have that it's Fourier cosine and sine transforms respectively are given by ~\cite{bateman_manuscript_vol1}.
\begin{align}
	F_p^C(\omega)&= p! \left(\frac{\mu}{\mu^2 + \omega^2} \right)^2 \sum_{m=0}^{\floor*{0.5(p+1)}} (-1)^m {p+1 \choose 2m} \left(\frac{\omega}{\mu}\right)^{2m}\\	
	F_p^S(\omega)&= p! \left(\frac{\mu}{\mu^2 + \omega^2} \right)^2 \sum_{m=0}^{\floor*{0.5p}} (-1)^m {p+1 \choose 2m + 1} \left(\frac{\omega}{\mu}\right)^{2m+1}
\end{align}
Thus for the subtractive term in Eq.~\eqref{eq::appendix::expanded_fourier_representation},
\begin{align}
	F_{C,p}(\omega) - j F_{S,p}(\omega) &= p! \left(\frac{\mu}{\mu^2 + \omega^2} \right)^2 \sum_{m=0}^{p+1} (-j)^m {p+1 \choose m} \left(\frac{\omega}{\mu}\right)^{m}\\
	&= p! \left(\frac{\mu}{\mu^2 + \omega^2} \right)^2\left (1 - j \frac{\omega}{\mu}\right)^{p+1}
\end{align}
and similarily for the additive term in Eq.~\eqref{eq::appendix::expanded_fourier_representation}
\begin{align}
	F_{C,p}(\omega) + j F_{S,p}(\omega) = p! \left(\frac{\mu}{\mu^2 + \omega^2} \right)^2 \left(1 + j \frac{\omega}{\mu}\right)^{p+1}
\end{align}
Using these expansions Eq.~\eqref{eq::appendix::expanded_fourier_representation} then becomes
\begin{align*}
	2\left(\norm{F_{C,p}(\omega)}^2 + \norm{F_{S,p}(\omega)}^2\right) =& (p!)^2 \left[\left(\frac{\mu}{\mu^2 + \omega^2} \right)\left(\frac{\mu}{\mu^2 + \omega^2} \right)^{\ast}\,\,\right]^{p+1}\\
	&\times \left[\left[\left(1-j\frac{\omega}{\mu}\right)\left(1-j\frac{\omega}{\mu}\right)^{\ast}\,\,\right]^{p+1} + \left[\left(1+j\frac{\omega}{\mu}\right)^{\ast}\left(1+j\frac{\omega}{\mu}\right)\right]^{p+1}\right]
\end{align*}
which upon some simplification we obtain
\begin{align}
	S(\omega) &= \frac{1}{2}(k!)^2 \left(a^2 + b^2\right)^{p+1} \frac{(a^2 + (\omega - jb)^2)^{p+1} + (a^2 + (\omega + jb)^2)^{p+1}}{\left[(a^2+b^2)^4  + (a^2-b^2)\omega^2 + \omega^4\right]^{p+1}}\\
	&= \frac{1}{2}(p!)^2 \frac{(a^2 + (\omega - b)^2)^{p+1} + (a^2 + (\omega + b)^2)^{p+1}}{\left[(a^2+b^2-2b\omega + \omega^2)(a^2+b^2 + 2b\omega + \omega^2)\right]^{p+1}}\\
	&= \frac{1}{2}(p!)^2 \frac{(a^2 + (\omega - b)^2)^{p+1} + (a^2 + (\omega + b)^2)^{p+1}}{\left[(a^2 + (\omega - b)^2)(a^2 + (\omega + b)^2)\right]^{p+1}}\\
	&= \frac{1}{2}(p!)^2  \left(\left[\frac{1}{(a^2 + (\omega + b)^2)}\right]^{p+1} + \left[\frac{1}{(a^2 + (\omega - b)^2)}\right]^{p+1}\right)
\end{align}
where partial fraction simplification was used in the last line~\citep*{proakis_dsp_text_fourth}.
Inspection shows that this PSD is very similar in terms of each summand to the PSD of the Mat\'ern kernel.  Substituting $\zeta_1 = \omega + b$  and $\zeta_2 = \omega - b$ the inverse Fourier transform is easily found.

\subsection{Mixture of stationary \HM kernels are dense}\label{app::mixture_hm_dense}
\begin{theorem*}[Mixture of stationary \HM kernels are dense.]
	For any fixed $p$, \HM kernels are dense in the space of square integrable functions, hence they are dense with respect to $\mathcal{L}_2$ convergence.
\end{theorem*}
\begin{proof}
	From Wiener's Tauberian theorem, we have that if $S \in \mathcal{L}^2(\reals)$ is square integrable then the span of the translations $S(\omega + b)$ is dense in $\mathcal{L}^2(\reals)$ if and only if the real zeros of the Fourier transform of $S$ form a set of Lebesgue measure 0~\citep*{rudin_functional_analysis_91}.
	
	Take a Hida-Mat\'ern kernel, fix $a$ and $p$, then let $k_{\mathcal{H}}(\tau) = \sum_{i=1}^L c_i \, k_{H, p}(\tau ; \, a, b_i)$ and denote the PSD, or the Fourier transform of $k_{\mathcal{H}}(\tau)$ as  $S_{\mathcal{H}}(\omega) = \sum_{i=1}^L c_i \, S_{H,p}(\omega; \, a, b)$.  We can write, $k_{H,p}(\tau; a, b_i) = \exp(-jb_i\tau)k_{H, p}(\tau; \, a, 0) + \exp(jb_i\tau)k_{H, p}(\tau; \, a, 0) (\tau)$.  By the frequency shifting property of the Fourier transform we have that $S_{H, p}(\omega; \, a, b) = S_{H, p}(\omega - b; \, a, 0) + S_{H, p}(\omega + b; \, a, 0)$. 
	
	Furthermore, since $S_{H,p}(\omega; \, a, 0)$ is symmetric about the origin, we have that $\mathcal{F}[S_{H,p}(\omega; \, a, 0)] = \mathcal{F}^{-1}[S_{H,p}(\omega; \, a, 0)] $.  Recognizing that the second term results in the non-oscillatory Hida-Mat\'ern kernel/Mat\'ern kernel in the time domain now makes it clear that the Fourier transform of $S_{H,p}(\omega; \, a, 0)$ is strictly positive and so its real zeros have Lebesgue measure 0.
	
	Now, using Wiener's Tauberian theorem, we have that that the span of translations of $S_{H,p}(\omega)$ is dense in $\mathcal{L}^2(\reals)$.  So, if we have some square integrable kernel $k(\tau)$ with Fourier transform $S(\omega)$, then from Wiener's Tauberian theorem we should be able to find a linear combination of Hida-Mat\'erns such that
	\begin{align}
		\left(\int \mid S(\omega) - \sum_{i=1}^L c_i \, S_{H,p}(\omega; \, a, b) \mid^2 d\omega\right)^{\tfrac{1}{2}} \rightarrow 0
	\end{align}
	However, by now using Parseval's theorem we also have that 
	\begin{align}
		\int \mid S(\omega) - \sum_{i=1}^L c_i \, S_{H,p}(\omega; \, a, b) \mid^2 d\omega = \int \mid k(\tau) - \sum_{i=1}^L c_i \, k_{H,p}(\tau; \, a, b) \mid^2 d\tau
	\end{align}
	which also means that the class of Hida-Mat\'ern kernels are dense in the space of $\mathcal{L}^2(\reals)$.
\end{proof}
\begin{remark*}
	We note a similar way of proving pointwise convergence for any stationary, real valued, positive semidefinite kernel modulated by a cosine was used in~\cite{samoGeneralizedSpectralKernels2015}.
\end{remark*}

\subsection{Stationary Covariance of an $N$-ple Markov GP}\label{app::stationary_cov_hm_gp}
Though the fact that the stationary covariance of the continuous time representation of a GP as in Eq.~\eqref{eq::main_paper::difference_eq} being $\Ks(0)$ is intuitive, it is also true that the stationary covariance of the discretized model as presented for standard GP regression is $\Ks(0)$ -- independent of the spacing of the observations.
\label{prop::Pinf-is-K0}
\begin{proposition}
	For state space models as defined, the stationary marginal covariance of the process, $\fs(t)$,  denoted $\Pinf$ is exactly $\Ks(0)$.
\end{proposition}
\begin{proof}
	The stationary marginal covariance satisfies the discrete Lyapunov equation,
	\begin{equation}
		\vA(\Delta) \Pinf \vA(\Delta)^\top + \vQ(\Delta) - \Pinf = \vzero
	\end{equation}
	By substituting $\Ks(\Delta)\Ks(0)^{-1}$ for $\vA(\Delta)$ and expanding $\vQ(\Delta)$ we get that
	\begin{equation}
		\Ks(\Delta)\Ks(0)^{-1} \Pinf \Ks(0)^{-1} \Ks(\Delta)^\top + \Ks(0) - \Ks(\Delta)\Ks(0)^{-1}\Ks(\Delta)^\top - \Pinf = \vzero
	\end{equation}
	which can be factored as
	\begin{align}
		\Ks(\Delta)(\Ks(0)^{-1} \Pinf \Ks(0)^{-1} - \Ks(0)^{-1} ) \Ks(\Delta)^\top + \Ks(0) - \Pinf = \vzero
	\end{align}
	Letting $\vY = \Ks(0)^{-1} \Pinf \Ks(0)^{-1} - \Ks(0)^{-1}$ gives us that $\Pinf = \Ks(0)(\vY + \Ks(0)^{-1})\Ks(0)$ upon whose substition we find
	\begin{align}
		\Ks(\Delta) \vY \Ks(\Delta)^\top+ \Ks(0) - \Ks(0)\vY\Ks(0) - \Ks(0) &= \vzero\\
		\label{eq::proof::Pinf_K0_in_Y}
		\Ks(\Delta)\vY \Ks(\Delta)^\top - \Ks(0)\vY \Ks(0) &= \vzero
	\end{align}
	Vectorizing Eq.~\ref{eq::proof::Pinf_K0_in_Y} we now find that
	\begin{align}
		\left(\Ks(\Delta) \otimes \Ks(\Delta)^\top - \Ks(0) \otimes \Ks(0)\right) \vect({\vY} )= \vzero
	\end{align}
	which means that since a unique solution of $\vect(\vY)$ must exist, then the only possibility is that $\vect(\vY) = \vzero$ since its premultiplier is of full rank and has no nullspace. This then means that
	\begin{align}
		\vect(\vY) &= \vzero\\
		\vY &= \vzero\\
		\Ks(0)^{-1} \Pinf \Ks(0)^{-1} - \Ks(0)^{-1} &= \vzero\\
	\end{align}
	giving us the result
	\begin{equation}
		\Pinf = \Ks(0)	
	\end{equation}
	
	\begin{remark}
		We see that the stationary covariance is invariant to the choice of $\Delta$, as such we could have used the continuous or discrete lyapunov equations to solve for the stationary covariance.  Indeed, plugging in this solution to the continuous time Lyapunov equation is consistent.
	\end{remark}
\end{proof}

\subsection{Numerically stable Kalman updates}\label{app::kalman_updates}
Taking advantage of the sparsity of the observation extraction vector $\vh$ can also aid in reducing numerical noise by recognizing that its sparsity leads to low rank Kalman updates of the predicted covariance and simplified equations for the updates of the mean.

\textbf{Kalman Equations}
Take $\vm_k^-$ and $\vP_k^-$ to be the mean and covariance prediction at $t_k$ with $\vm_k$ and $\vP_k$ to be the updated mean and covariance and $t_k$.  Then with $\vA(\Delta_k) = \Ks(t_{k+1} -  t_k)\Ks(0)^{-1}$ and $\vQ(\Delta_k) = \Ks(0) - \Ks(t_{k+1} -  t_k) \Ks(0)^{-1} \Ks(t_{k+1} -  t_k)^H$ the Kalman recursions follow 

\begin{align*}
	\vm_k^- &= \vA(\Delta_k) \vm_{k-1}\\
	\vP_k^- &= \vQ(\Delta_k) + \vA(\Delta_k) \vP_k \vA(\Delta_k)^\top\\
	\nu_k &= y_k - \vh^\top\vm_k^-\\
	S_k & = \vh^\top \vP_k^- \vh + R\\
	\mathcal{K}_k &= \vP_k^- \vh S_k^{-1}\\
	\\
	\vm_k &= \vm_k^- + \mathcal{K}_k \nu_k\\
	\vP_k  &= \vP_k^- -\mathcal{K}_k \vh^\top \vP_k^-\\
	&= (\vI - \mathcal{K}_k\vh^\top)\vP_k^- (I - \mathcal{K}_k\vh^\top)^\top  + \mathcal{K}_k R \mathcal{K}_k^\top \quad\quad \text{(Joseph form)}
\end{align*} 
In general we would not use the standard covariance update because numerically it will not guarantee $\vP_k$ is PSD.  With that said, let $\mathcal{Z}$ be the set of indices where $\vh$ has non-zero elements and let's first expand the equation for $\vm_k$
\begin{align*}
	\vm_k &= \vm_k^- + \mathcal{K}_k \nu_k\\
	&= \vm_k^- + \vP_k^-\vh S_k^{-1} \nu_k\\
	&= \vm_k^- + \vP_k^- \vh  (\vh^\top \vP_k^- \vh + R)^{-1} \nu_k\\
	&= \vm_k^- + \vP_k^- \vh  (\vh^\top \vP_k^-\vh + R)^{-1} (y_k - \vh^\top \vm_k^-)\\
	&= \vm_k^- + \sum_{i \in \mathcal{Z}} \vP_k^-[:, i] \left(\sum_{i , j \in \mathcal{Z}}\vP_k^-[i, j] + R\right)^{-1} \left(y_k - \sum_{i \in \mathcal{Z}} \vm_k^-[i]\right)
\end{align*}
Now take $\alpha = \left(\sum_{i , j \in \mathcal{Z}}\vP_k^-[i, j] + R\right)^{-1}$ and $\beta = \left(y_k - \sum_{i \in \mathcal{Z}} \vm_k[i]\right)$ and we get that

\begin{align*}
	\vm_k &= \vm_k^- + \alpha\beta \sum_{i \in \mathcal{Z}} \vP_k^-[:, i]
\end{align*}
Making it obvious that it is sufficient to work with select columns and elements of $\vP_k^-$ and that the update is simply the sum of select scaled columns of $\vP_k^-$.

Let's do the same for the updated covariance,

\begin{align*}
	\vP_k &= \vP_k^- - \mathcal{K}_k \vh^\top \vP_k^-\\
	&= \vP_k^- - \vP_k^- \vh S_k^{-1} \vh^\top \vP_k^-\\
	&= \vP_k^- - \vP_k^- \vh (\vh^\top \vP_k^- \vh + R)^{-1} \vh^\top \vP_k^-\\
	&= \vP_k^- - \alpha \vP_k^- \vh \vh^\top \vP_k^-\\
	&= \vP_k^- - \alpha \sum_{i \in \mathcal{Z}} \vP_k^-[:, i] \vP_k^-[:, i]^\top
\end{align*}
Now, it is obvious that the updated covariance $\vP_k$ consists of subtracting a $\text{rank}(|\mathcal{Z}|)$ matrix from $\vP_k^-$ which is simply the sum of  outer products of select columns of $\vP_k^-$.


\newpage
	
	
	
\newpage
\newpage
\begin{thebibliography}{53}
    \providecommand{\natexlab}[1]{#1}
    \providecommand{\url}[1]{\texttt{#1}}
    \expandafter\ifx\csname urlstyle\endcsname\relax
      \providecommand{\doi}[1]{doi: #1}\else
      \providecommand{\doi}{doi: \begingroup \urlstyle{rm}\Url}\fi
    
    \bibitem[Adler(2010)]{adler_random_fields}
    Robert~J Adler.
    \newblock \emph{The Geometry of Random Fields}.
    \newblock Society for Industrial and Applied Mathematics, January 2010.
    \newblock \doi{10.1137/1.9780898718980}.
    \newblock URL \url{https://doi.org/10.1137/1.9780898718980}.
    
    \bibitem[Anderson and Moore(1979)]{anderson_moore_state_space_book}
    Brian D.~O Anderson and John~B Moore.
    \newblock \emph{Optimal Filtering}.
    \newblock {Prentice-Hall}, {Englewood Cliffs, N.J.}, 1979.
    \newblock ISBN 978-0-13-638122-8.
    
    \bibitem[Bateman(1954)]{bateman_manuscript_vol1}
    Harry Bateman.
    \newblock \emph{Tables of {{Integral Transforms Volume}} 1}.
    \newblock Bateman {{Manuscript Project}}. {McGraw-Hill}, 1st edition, 1954.
    \newblock URL \url{https://books.google.com/books?id=HfZQAAAAMAAJ}.
    
    \bibitem[Bui et~al.(2017)Bui, Nguyen, and Turner]{bui_streaming_gps_2017}
    Thang~D Bui, Cuong Nguyen, and Richard~E Turner.
    \newblock Streaming sparse gaussian process approximations.
    \newblock In I.~Guyon, U.~V. Luxburg, S.~Bengio, H.~Wallach, R.~Fergus,
      S.~Vishwanathan, and R.~Garnett, editors, \emph{Advances in Neural
      Information Processing Systems}, volume~30. Curran Associates, Inc., 2017.
    \newblock URL
      \url{https://proceedings.neurips.cc/paper/2017/file/f31b20466ae89669f9741e047487eb37-Paper.pdf}.
    
    \bibitem[Chang et~al.(2020)Chang, Wilkinson, Khan, and
      Solin]{solin_vi_ssm_gp_20}
    Paul~E. Chang, William~J. Wilkinson, Mohammad~Emtiyaz Khan, and Arno Solin.
    \newblock Fast variational learning in state-space gaussian process models.
    \newblock \emph{CoRR}, abs/2007.04731, 2020.
    \newblock URL \url{https://arxiv.org/abs/2007.04731}.
    
    \bibitem[Corenflos et~al.(2021)Corenflos, Zhao, and
      S{\"a}rkk{\"a}]{sarkka_log_time_gp}
    Adrien Corenflos, Zheng Zhao, and Simo S{\"a}rkk{\"a}.
    \newblock Temporal {{Gaussian Process Regression}} in {{Logarithmic Time}}.
    \newblock \emph{arXiv:2102.09964 [cs, stat]}, May 2021.
    
    \bibitem[Dolph and Woodbury(1952)]{woodbury_greens_fn_and_gps_1952}
    C.~L. Dolph and M.~A. Woodbury.
    \newblock On the relation between green's functions and covariances of certain
      stochastic processes and its application to unbiased linear prediction.
    \newblock \emph{Transactions of the American Mathematical Society}, 72\penalty0
      (3):\penalty0 519--519, March 1952.
    \newblock \doi{10.1090/s0002-9947-1952-0050215-4}.
    \newblock URL \url{https://doi.org/10.1090/s0002-9947-1952-0050215-4}.
    
    \bibitem[Doob(1944)]{doob_1944}
    J.~L. Doob.
    \newblock The elementary gaussian processes.
    \newblock \emph{The Annals of Mathematical Statistics}, 15\penalty0
      (3):\penalty0 229--282, September 1944.
    \newblock \doi{10.1214/aoms/1177731234}.
    \newblock URL \url{https://doi.org/10.1214/aoms/1177731234}.
    
    \bibitem[Gardner et~al.(2018)Gardner, Pleiss, Weinberger, Bindel, and
      Wilson]{gpytorch_python}
    Jacob Gardner, Geoff Pleiss, Kilian~Q Weinberger, David Bindel, and Andrew~G
      Wilson.
    \newblock Gpytorch: Blackbox matrix-matrix gaussian process inference with gpu
      acceleration.
    \newblock In S.~Bengio, H.~Wallach, H.~Larochelle, K.~Grauman, N.~Cesa-Bianchi,
      and R.~Garnett, editors, \emph{Advances in Neural Information Processing
      Systems}, volume~31. Curran Associates, Inc., 2018.
    \newblock URL
      \url{https://proceedings.neurips.cc/paper/2018/file/27e8e17134dd7083b050476733207ea1-Paper.pdf}.
    
    \bibitem[Hartikainen and Sarkka(2010)]{sarkka_hartikainen_ssms_2010}
    Jouni Hartikainen and Simo Sarkka.
    \newblock Kalman filtering and smoothing solutions to temporal {{Gaussian}}
      process regression models.
    \newblock In \emph{2010 {{IEEE International Workshop}} on {{Machine Learning}}
      for {{Signal Processing}}}, pages 379--384. {IEEE}, 2010.
    \newblock ISBN 978-1-4244-7875-0.
    \newblock \doi{10.1109/MLSP.2010.5589113}.
    \newblock URL \url{http://ieeexplore.ieee.org/document/5589113/}.
    
    \bibitem[Hida(1960)]{hida_canonical_representation_1960}
    Takeyuki Hida.
    \newblock {Canonical representations of Gaussian processes and their
      applications}.
    \newblock \emph{Memoirs of the College of Science, University of Kyoto. Series
      A: Mathematics}, 33\penalty0 (1):\penalty0 109 -- 155, 1960.
    \newblock \doi{10.1215/kjm/1250776062}.
    \newblock URL \url{https://doi.org/10.1215/kjm/1250776062}.
    
    \bibitem[Hida and Hitsuda(1993)]{hida_gp_book}
    Takeyuki Hida and Masuyuki Hitsuda.
    \newblock \emph{Gaussian {{Processes}}}.
    \newblock {American Mathematical Society}, 1993.
    
    \bibitem[Hida and Si(2004)]{hida_sisi_white_noise_random_fields}
    Takeyuki Hida and Si~Si.
    \newblock \emph{An Innovation Approach to Random Fields}.
    \newblock {WORLD} {SCIENTIFIC}, July 2004.
    \newblock \doi{10.1142/5046}.
    \newblock URL \url{https://doi.org/10.1142/5046}.
    
    \bibitem[Jazwinski(2007)]{Jazwinski2007-yx}
    Andrew~H Jazwinski.
    \newblock \emph{Stochastic Processes and Filtering Theory}.
    \newblock Courier Corporation, January 2007.
    \newblock ISBN 9780486462745.
    \newblock URL
      \url{https://play.google.com/store/books/details?id=4AqL3vE2J-sC}.
    
    \bibitem[Karaletsos and Bui(2020)]{bui_hieararchical_gps}
    Theofanis Karaletsos and Thang~D Bui.
    \newblock Hierarchical gaussian process priors for bayesian neural network
      weights.
    \newblock In H.~Larochelle, M.~Ranzato, R.~Hadsell, M.~F. Balcan, and H.~Lin,
      editors, \emph{Advances in Neural Information Processing Systems}, volume~33,
      pages 17141--17152. Curran Associates, Inc., 2020.
    \newblock URL
      \url{https://proceedings.neurips.cc/paper/2020/file/c70341de2c112a6b3496aec1f631dddd-Paper.pdf}.
    
    \bibitem[Karvonen and Sarkka(2016)]{karvonen_pade_approximations_15}
    Toni Karvonen and Simo Sarkka.
    \newblock Approximate state-space gaussian processes via spectral
      transformation.
    \newblock In \emph{2016 {IEEE} 26th International Workshop on Machine Learning
      for Signal Processing ({MLSP})}. {IEEE}, September 2016.
    \newblock \doi{10.1109/mlsp.2016.7738812}.
    \newblock URL \url{https://doi.org/10.1109/mlsp.2016.7738812}.
    
    \bibitem[Kohn and Ansley(1987)]{ansley_spline_state_space_1987}
    Robert Kohn and Craig~F. Ansley.
    \newblock A new algorithm for spline smoothing based on smoothing a stochastic
      process.
    \newblock \emph{{SIAM} Journal on Scientific and Statistical Computing},
      8\penalty0 (1):\penalty0 33--48, January 1987.
    \newblock \doi{10.1137/0908004}.
    \newblock URL \url{https://doi.org/10.1137/0908004}.
    
    \bibitem[Kosambi(1943)]{kosambi_1943}
    D.~D. Kosambi.
    \newblock Statistics in function space.
    \newblock In \emph{D.D. Kosambi}, pages 115--123. Springer India, 1943.
    \newblock \doi{10.1007/978-81-322-3676-4_15}.
    \newblock URL \url{https://doi.org/10.1007/978-81-322-3676-4_15}.
    
    \bibitem[Krämer and Hennig(2020)]{kramer_stable_probabilistic_odes_2020}
    Nicholas Krämer and Philipp Hennig.
    \newblock Stable implementation of probabilistic ode solvers, 2020.
    
    \bibitem[Kuss and Rasmussen(2005)]{rasmussen_classification_approx_gps}
    Malte Kuss and Carl~Edward Rasmussen.
    \newblock Assessing approximate inference for binary gaussian process
      classification.
    \newblock \emph{Journal of Machine Learning Research}, 6\penalty0
      (57):\penalty0 1679--1704, 2005.
    \newblock URL \url{http://jmlr.org/papers/v6/kuss05a.html}.
    
    \bibitem[Lee(1990)]{lee_1990_random_fields}
    Ke-Seung Lee.
    \newblock White noise approach to gaussian random fields.
    \newblock \emph{Nagoya Mathematical Journal}, 119:\penalty0 93--106, September
      1990.
    \newblock \doi{10.1017/s0027763000003135}.
    \newblock URL \url{https://doi.org/10.1017/s0027763000003135}.
    
    \bibitem[Loper et~al.(2021)Loper, Blei, Cunningham, and
      Paninski]{cunningham_leg_kernel_2020}
    Jackson Loper, David Blei, John~P. Cunningham, and Liam Paninski.
    \newblock Linear-time inference for gaussian processes on one dimension, 2021.
    
    \bibitem[Lévy(1951)]{levy_wiener_random_functions_1951}
    Paul Lévy.
    \newblock Wiener's {{Random Function}}, and {{Other Laplacian Random
      Functions}}.
    \newblock \emph{Proceedings of the Second Berkeley Symposium on Mathematical
      Statistics and Probability}, pages 171--187, 1951.
    \newblock URL
      \url{https://projecteuclid.org/ebooks/berkeley-symposium-on-mathematical-statistics-and-probability/Proceedings-of-the-Second-Berkeley-Symposium-on-Mathematical-Statistics-and/chapter/Wieners-Random-Function-and-Other-Laplacian-Random-Functions/bsmsp/1200500228}.
    
    \bibitem[Lévy(1956)]{levy_canonical_gps}
    Paul Lévy.
    \newblock A special problem of brownian motion, and a general theory of
      gaussian random functions.
    \newblock In Jerzy Neyman, editor, \emph{Contributions to {{Probability
      Theory}}}, pages 133--176. {University of California Press}, 1956.
    \newblock ISBN 978-0-520-35067-0.
    \newblock \doi{10.1525/9780520350670-013}.
    \newblock URL
      \url{https://www.degruyter.com/document/doi/10.1525/9780520350670-013/html}.
    
    \bibitem[McKean(1963)]{mckean_brownian_motion_several_dim_63}
    H.~P. McKean, jr.
    \newblock Brownian motion with a several-dimensional time.
    \newblock \emph{Theory of Probability \& Its Applications}, 8\penalty0
      (4):\penalty0 335--354, 1963.
    \newblock \doi{10.1137/1108042}.
    \newblock URL \url{https://doi.org/10.1137/1108042}.
    
    \bibitem[Ng et~al.(2018)Ng, Colombo, and Silva]{ng_gps_graphs_semisupervised}
    Yin~Cheng Ng, Nicol\`{o} Colombo, and Ricardo Silva.
    \newblock Bayesian semi-supervised learning with graph gaussian processes.
    \newblock In S.~Bengio, H.~Wallach, H.~Larochelle, K.~Grauman, N.~Cesa-Bianchi,
      and R.~Garnett, editors, \emph{Advances in Neural Information Processing
      Systems}, volume~31. Curran Associates, Inc., 2018.
    \newblock URL
      \url{https://proceedings.neurips.cc/paper/2018/file/1fc214004c9481e4c8073e85323bfd4b-Paper.pdf}.
    
    \bibitem[Nordsieck(1962)]{nordsieck_ode_transform_1962}
    Arnold Nordsieck.
    \newblock On {{Numerical Integration}} of {{Ordinary Differential Equations}}.
    \newblock \emph{Mathematics of Computation}, 16\penalty0 (77):\penalty0 22--49,
      1962.
    \newblock ISSN 0025-5718.
    \newblock \doi{10.2307/2003809}.
    
    \bibitem[Oksendal(1992)]{oksendal_sde_intro_text}
    Bernt Oksendal.
    \newblock \emph{Stochastic {{Differential Equations}} (3rd {{Ed}}.): {{An
      Introduction}} with {{Applications}}}.
    \newblock {Springer-Verlag}, 1992.
    \newblock ISBN 3387533354.
    
    \bibitem[Oppenheim and Schafer(2014)]{oppenheim_dsp_14}
    Alan~V. Oppenheim and Roland~W. Schafer.
    \newblock \emph{Discrete-Time Signal Processing}.
    \newblock Prentice Hall Signal Processing Series. {Pearson}, 3rd edition, 2014.
    \newblock ISBN 978-1-292-02572-8.
    
    \bibitem[Osborne(1966)]{osborne_nordsieck_1966}
    M.~R. Osborne.
    \newblock On {{Nordsieck}}'s method for the numerical solution of ordinary
      differential equations.
    \newblock \emph{BIT Numerical Mathematics}, 6\penalty0 (1):\penalty0 51--57,
      1966.
    \newblock ISSN 1572-9125.
    \newblock \doi{10.1007/BF01939549}.
    \newblock URL \url{https://doi.org/10.1007/BF01939549}.
    
    \bibitem[Pitt(1971)]{pitt_multidimensional_markov_71}
    Loren~D Pitt.
    \newblock A {{Markov}} property for {{Gaussian}} processes with a
      multidimensional parameter.
    \newblock \emph{Arch. Rational Mech. Anal. Archive for Rational Mechanics and
      Analysis}, 43\penalty0 (5):\penalty0 367--391, 1971.
    \newblock ISSN 0003-9527.
    
    \bibitem[Posa(2021)]{posa_difference_of_covariance_functions}
    Donato Posa.
    \newblock Models for the difference of continuous covariance functions.
    \newblock \emph{Stochastic Environmental Research and Risk Assessment},
      35\penalty0 (7):\penalty0 1369--1386, February 2021.
    \newblock \doi{10.1007/s00477-020-01947-1}.
    \newblock URL \url{https://doi.org/10.1007/s00477-020-01947-1}.
    
    \bibitem[Proakis and Manolakis(2007)]{proakis_dsp_text_fourth}
    John~G. Proakis and Dimitris~G. Manolakis.
    \newblock \emph{Digital Signal Processing}.
    \newblock {Pearson Prentice Hall}, 4th ed edition, 2007.
    \newblock ISBN 978-0-13-187374-2.
    
    \bibitem[Rahimi and Recht(2007)]{random_fourier_rahimi}
    Ali Rahimi and Benjamin Recht.
    \newblock Random features for large-scale kernel machines.
    \newblock In \emph{Proceedings of the 20th International Conference on Neural
      Information Processing Systems}, NIPS'07, pages 1177--1184, Red Hook, NY,
      USA, 2007. Curran Associates Inc.
    \newblock ISBN 9781605603520.
    
    \bibitem[Rasmussen and Williams(2005)]{Rasmussen2005-mq}
    Carl~E Rasmussen and Christopher K~I Williams.
    \newblock \emph{Gaussian Processes for Machine Learning}.
    \newblock Adaptive Computation and Machine Learning. The MIT Press, November
      2005.
    \newblock ISBN 9780262182539.
    
    \bibitem[Rudin(1991)]{rudin_functional_analysis_91}
    Walter Rudin.
    \newblock \emph{Functional Analysis}.
    \newblock McGraw-Hill, Inc, 1991.
    \newblock ISBN 978-0-07-054236-5 978-0-07-100944-7 978-7-111-13415-2
      978-0-07-061988-3.
    
    \bibitem[Samo and
      Roberts(2015{\natexlab{a}})]{samoGeneralizedSpectralKernels2015}
    Yves-Laurent~Kom Samo and Stephen Roberts.
    \newblock Generalized {{Spectral Kernels}}.
    \newblock 2015{\natexlab{a}}.
    \newblock URL \url{http://arxiv.org/abs/1506.02236}.
    
    \bibitem[Samo and Roberts(2015{\natexlab{b}})]{samo_p_markov_2015}
    Yves-Laurent~Kom Samo and Stephen~J. Roberts.
    \newblock p-markov gaussian processes for scalable and expressive online
      bayesian nonparametric time series forecasting, 2015{\natexlab{b}}.
    
    \bibitem[S{\"a}rkk{\"a}(2011)]{Sarkka2011-do}
    Simo S{\"a}rkk{\"a}.
    \newblock Linear operators and stochastic partial differential equations in
      gaussian process regression.
    \newblock In \emph{Artificial Neural Networks and Machine Learning -- {ICANN}
      2011}, pages 151--158. Springer Berlin Heidelberg, 2011.
    \newblock \doi{10.1007/978-3-642-21738-8\_20}.
    \newblock URL \url{http://dx.doi.org/10.1007/978-3-642-21738-8_20}.
    
    \bibitem[Solin(2016)]{solin_thesis_2016}
    A~Solin.
    \newblock \emph{Stochastic Differential Equation Methods for Spatio-Temporal
      Gaussian Process Regression}.
    \newblock PhD thesis, Aalto University, 2016.
    
    \bibitem[Solin and Särkkä(2014)]{solin_periodic_state_space}
    Arno Solin and Simo Särkkä.
    \newblock {Explicit Link Between Periodic Covariance Functions and State Space
      Models}.
    \newblock In Samuel Kaski and Jukka Corander, editors, \emph{Proceedings of the
      Seventeenth International Conference on Artificial Intelligence and
      Statistics}, volume~33 of \emph{Proceedings of Machine Learning Research},
      pages 904--912, Reykjavik, Iceland, 22--25 Apr 2014. PMLR.
    \newblock URL \url{http://proceedings.mlr.press/v33/solin14.html}.
    
    \bibitem[Solin et~al.(2018)Solin, Hensman, and
      Turner]{solin_infinite_horizon_gp}
    Arno Solin, James Hensman, and Richard~E Turner.
    \newblock Infinite-horizon gaussian processes.
    \newblock In S.~Bengio, H.~Wallach, H.~Larochelle, K.~Grauman, N.~Cesa-Bianchi,
      and R.~Garnett, editors, \emph{Advances in Neural Information Processing
      Systems}, volume~31. Curran Associates, Inc., 2018.
    \newblock URL
      \url{https://proceedings.neurips.cc/paper/2018/file/b865367fc4c0845c0682bd466e6ebf4c-Paper.pdf}.
    
    \bibitem[Stein(1999)]{stein_interpolation_textbook}
    Michael~L. Stein.
    \newblock \emph{Interpolation of {{Spatial Data}}}.
    \newblock Springer {{Series}} in {{Statistics}}. {Springer New York}, 1999.
    \newblock ISBN 978-1-4612-7166-6 978-1-4612-1494-6.
    \newblock \doi{10.1007/978-1-4612-1494-6}.
    \newblock URL \url{http://link.springer.com/10.1007/978-1-4612-1494-6}.
    
    \bibitem[Strang and MacNamara(2014)]{strang_t_plus_ah}
    Gilbert Strang and Shev MacNamara.
    \newblock Functions of difference matrices are toeplitz plus hankel.
    \newblock \emph{SIAM Review}, 56:\penalty0 525--546, 08 2014.
    \newblock \doi{10.1137/120897572}.
    
    \bibitem[Särkkä and Solin(2019)]{sarkka_applied_sdes_text}
    Simo Särkkä and Arno Solin.
    \newblock \emph{Applied {{Stochastic Differential Equations}}}.
    \newblock {Cambridge University Press}, 1 edition, 2019.
    \newblock ISBN 978-1-108-18673-5 978-1-316-51008-7 978-1-316-64946-6.
    \newblock \doi{10.1017/9781108186735}.
    \newblock URL
      \url{https://www.cambridge.org/core/product/identifier/9781108186735/type/book}.
    
    \bibitem[Titsias(2009)]{titsias_sparse_variational_paper}
    Michalis Titsias.
    \newblock Variational learning of inducing variables in sparse gaussian
      processes.
    \newblock In David van Dyk and Max Welling, editors, \emph{Proceedings of the
      Twelfth International Conference on Artificial Intelligence and Statistics},
      volume~5 of \emph{Proceedings of Machine Learning Research}, pages 567--574,
      Hilton Clearwater Beach Resort, Clearwater Beach, Florida USA, 16--18 Apr
      2009. PMLR.
    \newblock URL \url{http://proceedings.mlr.press/v5/titsias09a.html}.
    
    \bibitem[Tobar(2019)]{sinc_gp_bui_2018}
    Felipe Tobar.
    \newblock Band-limited gaussian processes: The sinc kernel.
    \newblock In H.~Wallach, H.~Larochelle, A.~Beygelzimer, F.~d'Alch\'{e} Buc,
      E.~Fox, and R.~Garnett, editors, \emph{Advances in Neural Information
      Processing Systems}, volume~32. Curran Associates, Inc., 2019.
    \newblock URL
      \url{https://proceedings.neurips.cc/paper/2019/file/ccce2fab7336b8bc8362d115dec2d5a2-Paper.pdf}.
    
    \bibitem[Wahba(1978)]{wahba_spline_smoothing_state_space_1978}
    Grace Wahba.
    \newblock Improper priors, spline smoothing and the problem of guarding against
      model errors in regression.
    \newblock \emph{Journal of the Royal Statistical Society: Series B
      (Methodological)}, 40\penalty0 (3):\penalty0 364--372, July 1978.
    \newblock \doi{10.1111/j.2517-6161.1978.tb01050.x}.
    \newblock URL \url{https://doi.org/10.1111/j.2517-6161.1978.tb01050.x}.
    
    \bibitem[Weinert and Sidhu(1978)]{weinert_part_1_smoothing_splines_1978}
    H.~Weinert and G.~Sidhu.
    \newblock A stochastic framework for recursive computation of spline
      functions--part i: Interpolating splines.
    \newblock \emph{IEEE Transactions on Information Theory}, 24\penalty0
      (1):\penalty0 45--50, 1978.
    \newblock \doi{10.1109/TIT.1978.1055825}.
    
    \bibitem[Weinert et~al.(1980)Weinert, Byrd, and
      Sidhu]{weinert_part_2_smoothing_splines_1980}
    H.~L. Weinert, R.~H. Byrd, and G.~S. Sidhu.
    \newblock A stochastic framework for recursive computation of spline functions:
      Part {II}, smoothing splines.
    \newblock \emph{Journal of Optimization Theory and Applications}, 30\penalty0
      (2):\penalty0 255--268, February 1980.
    \newblock \doi{10.1007/bf00934498}.
    \newblock URL \url{https://doi.org/10.1007/bf00934498}.
    
    \bibitem[Wilson and Adams(2013)]{Wilson2013-ud}
    A~Wilson and R~Adams.
    \newblock Gaussian process kernels for pattern discovery and extrapolation.
    \newblock \emph{International conference on machine}, 2013.
    \newblock URL \url{http://proceedings.mlr.press/v28/wilson13.html}.
    
    \bibitem[Wilson and Nickisch(2015)]{Wilson2015-ak}
    Andrew Wilson and Hannes Nickisch.
    \newblock Kernel interpolation for scalable structured gaussian processes
      ({KISS-GP}).
    \newblock In Francis Bach and David Blei, editors, \emph{Proceedings of the
      32nd International Conference on Machine Learning}, volume~37 of
      \emph{Proceedings of Machine Learning Research}, pages 1775--1784, Lille,
      France, 2015. PMLR.
    \newblock URL \url{http://proceedings.mlr.press/v37/wilson15.html}.
    
    \bibitem[Álvarez and
      Lawrence(2011)]{alvarezComputationallyEfficientConvolved2011}
    Mauricio~A. Álvarez and Neil~D. Lawrence.
    \newblock Computationally {{Efficient Convolved Multiple Output Gaussian
      Processes}}.
    \newblock \emph{Journal of Machine Learning Research}, 12\penalty0
      (41):\penalty0 1459--1500, 2011.
    \newblock URL \url{http://jmlr.org/papers/v12/alvarez11a.html}.
    
\end{thebibliography}
\end{document}